\documentclass{article}
\usepackage{iclr2027_conference}
\usepackage{iftex}
\ifPDFTeX
  \usepackage[T1]{fontenc}
  \usepackage{times}
\else
  \usepackage{fontspec}
\fi
\usepackage{amsmath,amsthm}
\usepackage{graphicx}
\usepackage{booktabs,tabularx}
\usepackage{url}
\usepackage{hyperref}
\newcommand{\byte}{\textsc{Byte}}
\newcommand{\subword}{\textsc{Subword}}
\newcommand{\bpb}{\mathrm{bpb}}
\newtheorem{proposition}{Proposition}


\title{Byte Language Models: Scaling, Emergent Abstractions, and Information Allocation}
\author{Jie Wang\textsuperscript{1}\thanks{Equal contribution.} ,
Shiwei Luo\textsuperscript{1}\footnotemark[1] ,
Qi Zhang\textsuperscript{2} ,
Yuanbin Wu\textsuperscript{1} \\
\textsuperscript{1}School of Computer Science, East China Normal University, Shanghai, China \\
\textsuperscript{2}School of Computer Science, Fudan University, Shanghai, China \\
\texttt{jiewang.cs@stu.ecnu.edu.cn} 
}
\iclrfinalcopy

\begin{document}
\maketitle
\lhead{Preprint}

\begin{abstract}

Tokenizer-free language models remove the inductive bias of fixed tokenizers by modeling text directly as bytes, but the resulting longer sequences substantially increase computation and eliminate explicit text abstractions. We ask whether this additional computation can be useful, and whether standard Transformers can learn the abstractions that tokenization provides. We study these questions on Transformers without specialized tokenization-related architectures. With token-superposition training and hash embeddings, byte Transformers consistently outperform subword Transformers as model size scales. We further find that byte Transformers build local text abstractions as external tokenizers: a set of segmentation-like positions are used to collect local context representations, and restricting up to $25\%$ of intermediate layers to these local representations preserves downstream performance. Finally, these learned structures induce highly non-uniform generation difficulty, with uncertainty concentrated near local structure boundaries; exploiting them for speculative decoding yields $3.4\times$ more accepted tokens than in subword Transformers.

\end{abstract}

\section{Introduction}
\label{sec:introduction}

Tokenizer-free language models, or byte language models, aim to remove externally imposed tokenizers by modeling text directly as byte sequences \citep{xue2022byt5,yu2023megabyte,slagle2024spacebyte,tsiamas-etal-2025-improving,shah2026dnahnet}. This formulation provides a universal input space and a new joint learning setting. However, it also introduces two main challenges. First, byte sequences are longer than subword sequences. If a typical subword contains K (usually around 4) bytes, a byte model must process K times more sequence positions. Second, prior studies suggest that effective language modeling benefits from prior abstractions over raw text. Directly modeling bytes can therefore underperform tokenized models, despite having a more general input representation. 

Recent works (BLT \citealp{pagnoni2025blt}, H-Net \citealp{hwang2025hnet}, ByteFlow \citealp{deng2026byteflow}) develop new network architectures with built-in tokenization mechanisms. For example, in BLT, local encoder layers first encode byte streams into variable-sized chunks, then global layers operate on the chunks, and local decoder layers decode chunks back to bytes. While these architectures demonstrate the promise of joint tokenization, they often need specialized training and optimization strategies, making it less clear how much of the observed gains are from the joint formulation.  In this work, we retain the \emph{flat} Transformer architecture (without the explicit local-global-local hierarchy) and investigate the two challenges underlying byte-level language modeling: \emph{Can the additional computation caused by longer sequences be beneficial, and how does a Transformer recover useful text abstractions directly from bytes?}

We study the scaling behavior of byte Transformers, and the main findings are: 1) with two modeling ingredients (token-superposition training \citep{peng2026tst} and hash embeddings), byte Transformers consistently outperform subword Transformers when scaling the number of parameters (Figure \ref{fig:isoflops-frontiers}a). 2) Byte models remain effective as training data continues to scale. Given a fixed FLOPs budget (Figure \ref{fig:isoflops-frontiers}c), the optimal proportion between data and parameters is around 1.5 to 3 times that of subword models. These results motivate a different perspective on sequence length: rather than treating longer sequences as a defect, they can provide additional computation under the same parameter budget, which is a useful scaling axis for language models. Further evaluations on downstream text-only and vision--language tasks show that byte Transformers significantly improve performance on tasks that require fine-grained perception (around $40\%$ relative improvements over subword Transformers on CUTE scores, $20\%$ on OCRBench~\citep{liu2023ocrbench}).

We next investigate how a byte Transformer develops the abstractions that tokenization provides explicitly. By studying an input-copying task on byte models, we find that copying actions need to attend only to certain segmentation positions (not necessarily whitespace), where information from local text pieces can be read out.  Based on this finding, we further constrain a range of intermediate layers (up to $25\%$ of all layers) to access only these locally aggregated representations, resembling a global modeling stage over learned local structures. The resulting model retains downstream performance without additional training. This suggests that the apparent hierarchy of local abstraction and global reasoning is not merely imposed by tokenizer-free architectures such as BLT and H-Net, but can also emerge implicitly in the byte Transformer.

Finally, we study whether these learned local structures can be exploited during inference. By plotting the per-byte loss distribution, we find that byte Transformers have much higher variance than subword models: many byte positions have near-zero loss (in bits per byte), while substantially larger losses concentrate near the boundaries of the learned local structures. This observation suggests that the extra generation steps are not uniformly expensive in terms of model uncertainty. We therefore apply speculative decoding to construct local structures during inference \citep{wang2024mambabyte,kallini2026fast}, and show that a drafting model on byte Transformers can produce $3.4 \times$ more accepted tokens than on subword Transformers.   

To summarize, our results support a simple view of tokenizer-free language modeling: longer byte sequences provide useful additional computation. At the same time, the model can learn useful local abstractions internally rather than relying on a fixed tokenizer. This perspective suggests sequence length, learned abstraction, and computation allocation as closely connected dimensions for future language-model design.

\section{Scaling Byte Transformer Language Models}
\label{sec:scaling}

We first ask how the additional tokens in byte sequences affect model
quality under parameter and compute constraints. We then examine which
task capabilities benefit from byte resolution and how shared depth and
sparse capacity change these trade-offs.

\subsection{Building Byte Models}
\label{sec:architecture}




We first establish a controlled comparison between byte- and subword-level Transformers. 
Both models are trained from scratch with a mixed attention model.
\footnote{We generally follow the setting of Qwen3.5 \citep{qwen35blog}, which interleaves Gated DeltaNet \citep{yang2024gdn} and Gated Softmax Attention \citep{qiu2025gated} in a $3:1$ ratio.
We occasionally call the architecture \emph{flat} Transformer, 
in contrast to architectures with an explicit local-global-local hierarchy.}
The byte model uses a vocabulary of size $256$, 
while the subword model uses a $49,152$ token vocabulary from the cosmo2-tokenizer \citep{benallal2025smollm2}. 
We use ``token'' to refer to both byte and subword, depending on the model in context. 


Directly applying the standard training recipe to bytes is suboptimal. Byte sequences contain substantially more tokens, while each byte token carries less local semantic context than a subword token. We find the following two ingredients are effective for training byte models.

First, we use token-superposition training (TST; \citealp{peng2026tst}) to improve learning efficiency from the longer byte sequences. TST averages four consecutive token embeddings and predicts the corresponding next four-token bag. Applying TST during the first 30\% of training provides $4\times$ token exposure in this stage, while reverting to standard next-token prediction for the remaining 70\% results in $1.9\times$ total text exposure under the same number of training tokens (see Appendix \ref{app:TST} for details).

Second, we augment byte embeddings with hash embeddings as \citep{pagnoni2025blt}. 
The model learns representations for $n$-grams of bytes 
and combines them with byte embeddings, allowing each byte token to incorporate information from nearby multi-byte patterns
(Appendix \ref{app:hash_emb}). 
Similarly, we increase the convolution width of gated DeltaNet from $4$ (for subword models) to $16$ (for byte models), to build a broader local receptive field. 

Unless otherwise specified, byte models apply TST and hash embeddings
(denoted by \byte{}), 
while subword models don't (denoted by \subword{}). 
We later ablate these components to separate the contribution of the training recipe from that of input granularity.

\subsection{Parameter and Compute Efficiency under IsoFLOPs}
\label{sec:parameter-compute}


A longer byte sequence can be viewed either as a computational burden
or as an opportunity to allocate more computation to a model with the
same number of parameters.
To distinguish these effects, we study how byte and subword Transformers
scale along three coupled axes: model size, data exposure, and training
compute.

Following the standard practice of model scaling analyses, 
we train families of byte and subword models with different 
training FLOPs per token $M$ and 
cumulative training data volume $V$ (measured in bytes consumed during training),
record their validation loss $L$ (measured in bits per byte, BPB),
and fit the loss $L$ as function of $M$ and $V$.
We adapt the Chinchilla law \citep{hoffmann2022computeoptimal},
\begin{equation}
  L(M,V)=E+AM^{-\alpha}+BV^{-\beta},
  \label{eq:mv-law}
\end{equation}
where $\alpha, \beta, E, A, B$ are fiting coefficients. 
We also denote total training FLOPs $C=MD$, where $D$ is the amount of
training data in tokens.
\footnote{
A subword token or TST bag can represent multiple source bytes,
so we account for this using $V=sD$, with varying average source bytes per token $s$.}


Following \cite{hoffmann2022computeoptimal}, 
we compare model performance in the \emph{compute-optimal} manner.
Given a FLOPs budget $C$, we collect loss $L$ of 
models with various sizes $N$ (or equally, $M$) under $C$, 
fit a curve for the obtained $(N, L)$ pairs, 
and the optimal parameter count under the FLOPs budget 
$N^\star(C)$ is located at the minimum loss $L^\star$
(see Figure \ref{fig:isoflops} and Appendix \ref{app:isoflops} for details).
Similarly, we obtain the optimal training byte volume $V^\star(C)$. 
By varying $C$, we trace the 
$L^\star$--$N^\star$ and $L^\star$--$C$ Pareto frontiers
in Figure~\ref{fig:isoflops-frontiers}.

\begin{figure}[!t]
  \setlength{\abovecaptionskip}{2pt}
  \centering
  \includegraphics[width=\linewidth]{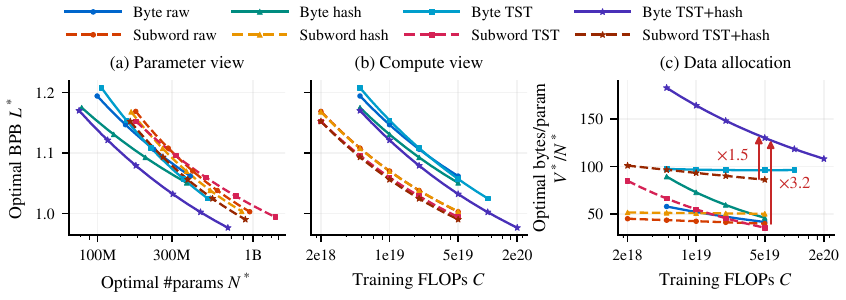}
  \caption{Pareto frontiers of compute-optimal models. (a) \byte{} has lower
    optimal BPB $L^\star$ at the same
    optimal parameter count $N^\star$, while (b) \subword{} has lower
    $L^\star$ at the same training compute $C$.
    (c) \byte{} optima generally allocate more source bytes per parameter.
    Markers indicate fitted optima at the evaluated compute budgets
    (see details in Appendix~\ref{app:isoflops}).}
  \label{fig:isoflops-frontiers}
\end{figure}

\paragraph{\byte{} has lower optimal loss at the same optimal parameter count}
The $L^\star$--$N^\star$ Pareto frontiers 
(Figure~\ref{fig:isoflops-frontiers}a) of \byte{}
consistently below \subword{}. 
It suggests that Transformers can transform additional computation 
to performance gains. 
Meanwhile, Figure~\ref{fig:isoflops-frontiers}b shows a different story:
under the same FLOPs budgets, \subword{} is always better.
We defer discussion on scaling along computation to Section \ref{sec:allocation}.

\paragraph{\byte{} optima allocate more source bytes per parameter.}
Compared with \subword{},
at the same training-compute budgets, the fitted optima select smaller
\byte{} models and generally allocate more ($3.2 \times$) source bytes per parameter
(Figure~\ref{fig:isoflops-frontiers}c).
It implies that \byte{} can reach lower BPB if training data continues to scale.

\paragraph{TST and hashing embeddings provide larger combined gains for \byte{}.}
We ablate TST and hashing at matched depth and training compute within
each family. Combining them consistently improves BPB over raw training,
with larger gains for \byte{}.

\subsection{Fine-Grained Capabilities in Text-Only and Vision--Language Tasks}
\label{sec:finegrained}

We next examine whether the modeling advantages of byte-level representations
translate into downstream text-only and vision-language tasks.
\footnote{
For Section~\ref{sec:finegrained} and \ref{sec:allocation}, we use a separate
\emph{main model suite} spanning 200M--3B sizes, instead of the \emph{IsoFLOPs sweep} settings in Section \ref{sec:parameter-compute}.
At each parameter scale, \byte{} and \subword{} use the same number of updates
and approximately the same source-text exposure. See Appendix~\ref{app:training} for detailed settings.
}
\footnote{
We train vision-language models in the LLaVA style \citep{liu2023llava}
with backbone language models replaced by \byte{}.
Training details are given in Appendix~\ref{app:vision-training}.
}
We evaluate word manipulation with CUTE \citep{edman2024cute}, 
and fine-grained
OCR tasks with OCRBench and TextVQA~\citep{singh2019textvqa}.

\begin{table}[!t]
  \centering
  \caption{\byte{} consistently outperforms \subword{} on fine-grained tasks across
model scales.}
  \label{tab:cute-scaling}
  \small
  \begin{tabular}{c cc cc cc}
    \toprule
    & \multicolumn{2}{c}{CUTE (\%; $\uparrow$)}
    & \multicolumn{2}{c}{OCRBench (\%; $\uparrow$)}
    & \multicolumn{2}{c}{TextVQA (\%; $\uparrow$)} \\
    \cmidrule(lr){2-3}
    \cmidrule(lr){4-5}
    \cmidrule(lr){6-7}
    Size
    & \subword{} & \byte{}
    & \subword{} & \byte{}
    & \subword{} & \byte{} \\
    \midrule
200M & 69.89 & \textbf{99.19}
     & 27.10 & \textbf{35.10}
     & 40.62 & \textbf{45.08} \\

400M & 70.78 & \textbf{99.18}
     & 27.40 & \textbf{37.40}
     & 43.82 & \textbf{48.37} \\

700M & 73.04 & \textbf{99.52}
     & 29.90 & \textbf{34.50}
     & 45.34 & \textbf{50.72} \\

1B & 76.44 & \textbf{99.56}
     & 31.50 & \textbf{39.40}
     & 47.50 & \textbf{52.36} \\

3B & 80.83 & \textbf{99.93}
     & 34.70 & \textbf{42.40}
     & 50.76 & \textbf{55.39} \\
 
    \bottomrule
  \end{tabular}
\end{table}

The results in Table~\ref{tab:cute-scaling} show that \byte{} has a clear
advantage in word manipulation capabilities over \subword{}.
We think the better fine-grained modeling ability of \byte{} 
comes from stronger
perception of inputs and stronger ability to generate byte-level sequences.
We particularly expect better alignment among modalities in future byte language model-based multimodality learning. 

We additionally evaluate both
model families on broader standard text-only ($21$ tasks, Table \ref{tab:qarc-main-details}) and vision-language benchmarks ($5$ tasks, Table \ref{tab:vision-all})
in Appendix~\ref{app:scaling}.
On text tasks, \byte{} model consistently outputperforms \subword{}, 
due to more computation per parameter. 
On vision-language tasks, \byte{} is on par with \subword{}.

\subsection{Controlling Computation with Shared Depth and Sparse Experts}
\label{sec:allocation}

The preceding results reveal a compute--quality trade-off between
\byte{} and \subword{}. We therefore examine two complementary ways to
shift this trade-off (Figure~\ref{fig:allocation}): increasing \subword{}
computation through shared depth, and reducing \byte{}'s activated
computation at a given parameter capacity with sparse mixture-of-experts
(MoE). Using the main model suite, we evaluate four-pass \subword{} models
and \byte{} MoE models under the training schedules of their dense
counterparts; architecture details are provided in
Appendix~\ref{app:architecture}.

\begin{figure}[htbp]
  \setlength{\abovecaptionskip}{2pt}
  \centering
  \includegraphics[width=\linewidth]{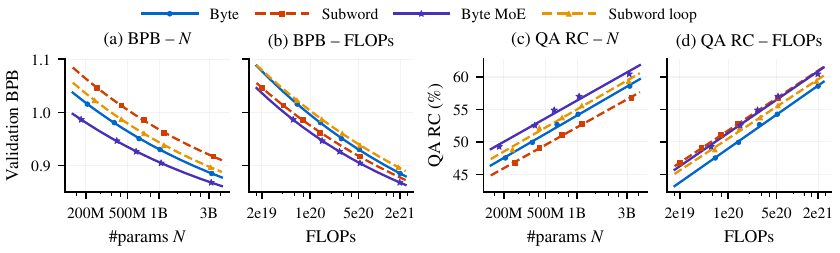}
  \caption{Compute scaling with shared depth and sparse experts.
Four-pass \subword{} closes most of the BPB gap to \byte{} with higher QA RC,
while sparse MoE lowers \byte{} BPB with less activated computation.
Markers are observations and curves are fits.}
  \label{fig:allocation}
\end{figure}

\paragraph{Increasing computation through shared depth.}
We follow the setting of LoopFormer \citep{jeddi2026loopformer}
to increase \subword{} computation with four shared backbone passes.
It closes
most of the BPB gap to \byte{} in both the parameter and compute views, showing that repeated depth is an effective
way to scale computation for subword models. 
The looped \subword{} models
also achieve consistently higher QA RC than the standard dense \byte{}
models across the evaluated scales.

\paragraph{Reducing activated computation with sparse experts.}
Sparse MoE \citep{deepseekv3}
consistently improves \byte{} BPB in both the parameter and
compute views, showing that sparse activation
is an effective way to reduce computation while retaining larger model
capacity. Notably, the 1B \byte{} MoE achieves lower BPB than the 3B dense
\subword{} model while reaching comparable QA RC.

Together, these results show that the byte--subword modeling gap depends
strongly on how computation is allocated. Increasing computation through
shared depth allows \subword{} models to recover much of the BPB gap,
while sparse activation improves the compute efficiency of \byte{} models
at comparable parameter capacity.

\section{Emergent Local Abstractions in Byte Transformers}
\label{sec:lgl}

Tokenization provides explicit local abstractions that can organize text into useful computational units. In this section, we investigate whether similar structures can emerge implicitly in byte Transformers without an explicit tokenizer or hierarchical architecture. We study this question through two complementary experiments. First, we examine how byte models organize information during prediction using an input-copying task. Second, we evaluate whether the representations revealed by this analysis can support alternative computation patterns through a frozen layer intervention.

\subsection{Copying as a probe of local information structure}

To probe how text semantics are abstracted and accessed in models, 
we study the copying task by prompting the model with ``Your task is to copy the string: $X$. Now copy it: $X$'', where $X$ is a short string composed of multiple meaningful English words.
We evaluate \byte{}-1B by analyzing attention patterns during the copying process, 
focusing on how different layers access the source text when generating the copied sequence (see Appendix~\ref{app:abstraction-copy} for details).

\begin{figure}
  \centering
  \includegraphics[width=\linewidth]{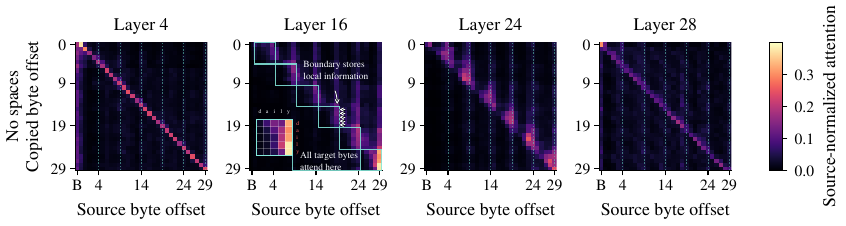}
  \caption{\textbf{\byte{}-1B attention maps during copying.}
Rows denote copied target bytes and columns denote source bytes from the
previous occurrence; each map shows target-to-source attention.
All test strings have internal spaces removed. Blue boxes indicate original
word chunks, and blue dashed lines mark their final bytes.
In intermediate layers, attention concentrates near chunk boundaries, and
multiple subsequent target bytes repeatedly attend to the same aggregated
positions, indicating the formation of meaningful local representations.}
  \label{fig:implicit-copy}
\end{figure}

As shown in Figure~\ref{fig:implicit-copy}, 
\byte{}-1B exhibits a clear attention aggregation pattern during copying.
In particular,
attention in intermediate layers becomes concentrated on a small set of 
source positions that are repeatedly accessed 
by multiple consecutive target bytes. 
This aggregation occurs even without any whitespace in inputs,
indicating that it does not simply arise from explicit separators. 
\footnote{Whitespaces may provide shortcuts during the aggregation: 
model can learn to utilize repeatedly appearing tokens, 
ablation of separators is in Figure \ref{fig:implicit-copy-separators}.}
The repeated reuse of these positions suggests that 
the byte model internally develops a mechanism for consolidating local information into representative states, 
providing an implicit structure over the byte sequence.

\subsection{Testing local abstractions with a frozen layer intervention}

The copying task reveals that multiple byte predictions can repeatedly access a small set of locally aggregated source representations. However, attention patterns alone do not establish whether these representations are sufficient for subsequent computation. We therefore introduce a frozen layer intervention that forces a contiguous range of intermediate layers to operate on a shorter sequence of locally aggregated representations, while preserving full byte-resolution states before and after the intervention. No model parameters are updated.

Concretely, following the external tokenizer applied in BLT \citep{pagnoni2025blt}, 
we partition the text sequence into local spans using a frozen, 
architecturally matched 200M \byte{} model and compress 
non-image text tokens (for vision--language tasks) 
by approximately \(4\times\). 
We treat the selected intermediate layers as a temporary global stage
(like the global model in BLT): 
each local span is mean-pooled into a single representation, 
processed at the compressed resolution, and then propagated back to the original byte-resolution sequence. This intervention introduces a coarse-grained computation stage into the pretrained \byte{}-1B model without modifying its parameters or retraining it. 
More details and results are in Appendix~\ref{app:abstraction-lgl}.

\begin{figure}
  \centering
  \includegraphics[width=\linewidth]{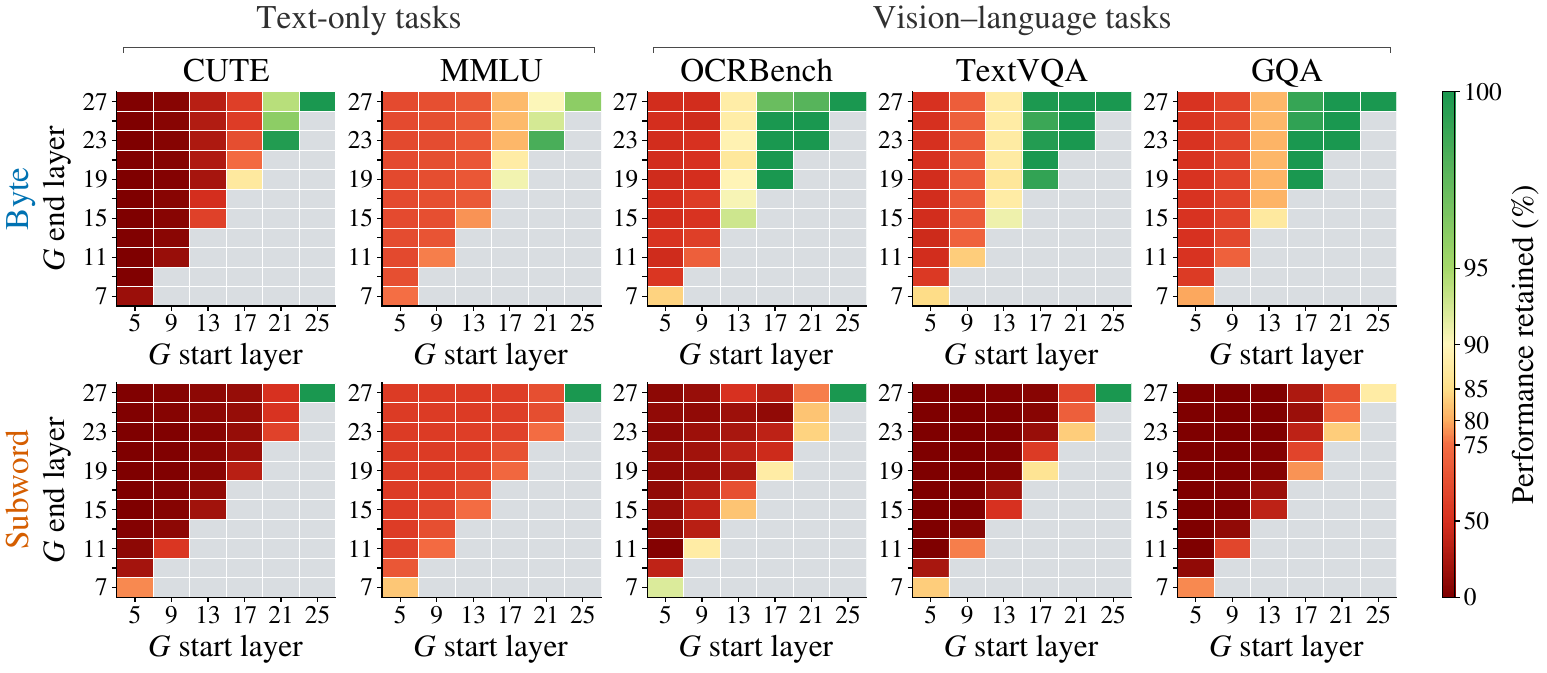}
  \caption{\textbf{Performance retention under frozen layer intervention.}
Each cell shows downstream performance relative to the uncompressed
baseline for a compressed layer interval \(G=[a,b]\).
\byte{} results are shown in the top row and \subword{} results in the
bottom row. 100\% indicates performance at or above the uncompressed baseline.
}
  \label{fig:implicit-lgl-sweep}
\end{figure}

We apply the same intervention to \byte{}-1B and \subword{}-1B across several downstream tasks, sweeping the start and end layers of the compressed global stage \(G\). As shown in Figure~\ref{fig:implicit-lgl-sweep}, \byte{} exhibits a clear recovery in performance when the global stage is placed sufficiently late in the network, with the trend becoming apparent for intervals starting around layer 17. More importantly, a stable compressible region emerges in the later layers: for \(G=[21,27]\), 7 of the model’s 28 layers (25\%) operate only on the compressed local representations, while downstream performance remains close to the uncompressed baseline across tasks. In contrast, applying the same intervention to the subword model causes substantial degradation across nearly all layer intervals. This contrast suggests that the robustness to an intermediate compressed global stage is a property of the byte model rather than a generic consequence of pooling hidden states.

\section{Information Allocation Along Text Sequences}
\label{sec:information}

Locally aggregated representations can support part of \byte{}'s
computation on a shorter sequence (Section~\ref{sec:lgl}), while the
model still predicts every byte separately. These additional modeling
steps raise the question of where the extra computation
(Section~\ref{sec:parameter-compute}) is useful and where it could be
reduced. For a given text source, deterministic, reversible tokenizations
preserve total entropy (Proposition~\ref{prop:recoding}), so finer
granularity distributes the same source information over more tokens.
We therefore examine how BPB varies across these tokens, how token
boundaries explain the distribution gap, and whether small drafts can
handle the low-BPB continuations.

\subsection{An Empirical BPB Distribution Law}
\label{sec:local-bpb}

\begin{figure}[!t]
  \setlength{\abovecaptionskip}{2pt}
  \centering
  \includegraphics[width=\linewidth]{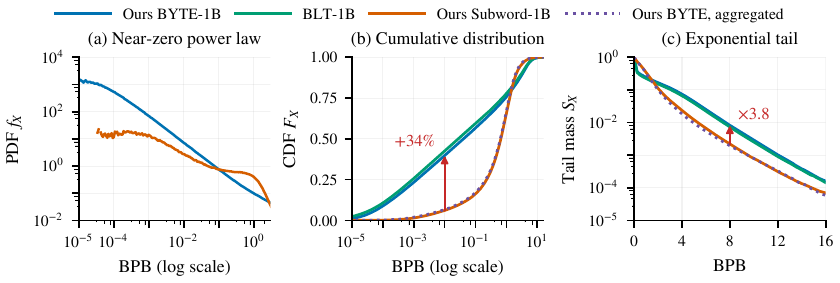}
  \caption{Token-span averaging largely closes the BPB distribution gap
    between \byte{} and \subword{}. The BPB distribution of \byte{} has more
    near-zero and high-BPB mass. All panels are byte-weighted.
    The dotted curve averages fixed
    \byte{} losses over \subword{} token spans, preserving mean BPB on
    complete spans (Appendix~\ref{app:bpb-aggregation}).}
  \label{fig:information}
\end{figure}

To compare local modeling difficulty across granularities, we use
\emph{byte-weighted BPB}: the random variable $X$ gives the BPB at a
uniformly sampled text byte. Following the BPB definition in
Appendix~\ref{app:evaluation-data}, a target token $T_i$ spanning $m_i$
original text bytes has BPB score $x_i$,
\begin{equation}
  x_i=-\frac{\ln p_\theta(T_i\mid\mathbf T_{<i})}{m_i\ln 2}.
  \label{eq:bpb-cost}
\end{equation}
In the distribution statistics, each token's BPB $x_i$ has weight
$m_i$, accounting for differences in token length. Our main comparison
uses the 1B \byte{} and \subword{} models, and we extend the analysis
to public models.

The BPB density $f_X$ exhibits a near-zero power law and an exponential
tail over finite ranges, connected by an empirical transition:
\begin{equation}
  f_X(x)\approx
  \begin{cases}
    k_{\rm pow}x^\alpha, & \epsilon<x\le u_1,\\
    f_X^{(\mathrm{trans})}(x), & u_1<x\le u_2,\\
    k_{\rm exp}e^{-x/\beta}, & u_2<x\le v.
  \end{cases}
  \label{eq:finite-density}
\end{equation}
Here $\alpha$ is the near-zero power-law exponent and $\beta$ is the
exponential decay scale. The fitted ranges are $(\epsilon,u_1]$ and
$(u_2,v]$; the transition $f_X^{(\mathrm{trans})}(x)$ and values outside
the fitted ranges remain empirical.\footnote{Appendix~\ref{app:bpb-protocol}
gives the evaluation settings, fitting procedure, and finite-range
diagnostics (Figure~\ref{fig:bpb-fits}).}
Let $F_X$ and $S_X$ denote the CDF and survival function of $X$,
with $S_X(x)=1-F_X(x)=\Pr(X>x)$ (Figure~\ref{fig:information}c).
We set the normalization coefficients $k_{\rm pow}$ and $k_{\rm exp}$
so that each fitted density integrates to the measured probability mass
in its fitting interval:
\begin{equation}
  k_{\rm pow}=\frac{F_X(u_1)-F_X(\epsilon)}
    {\int_\epsilon^{u_1}t^\alpha\,dt},
  \qquad
  k_{\rm exp}=\frac{F_X(v)-F_X(u_2)}
    {\int_{u_2}^{v}e^{-t/\beta}\,dt}.
  \label{eq:regime-amplitudes}
\end{equation}

\paragraph{Byte distributions place more mass at both low and high BPB.}
Near zero, byte densities approach $1/x$, with fitted power-law exponent
$\alpha\in[-0.97,-0.91]$, whereas subword densities are shallower, with
$\alpha\in[-0.8,-0.6]$
(Figure~\ref{fig:information}a--b). In the tail,
byte models have slightly larger decay scales $\beta$ and more
high-BPB mass overall than subword models
(Figure~\ref{fig:information}c). Together, these patterns show that
\byte{} has more probability mass at both low and high BPB than
\subword{} (Figure~\ref{fig:information}b--c; Table~\ref{tab:bpb-summary}).
We also verify this law in public byte models, including BLT and H-Net,
and widely used subword models. The byte models have nearly overlapping
BPB distributions (Figure~\ref{fig:bpb-models}; Table~\ref{tab:density}).

\begin{table}[htbp]
  \centering
  \caption{Byte models place more probability mass near zero and at high BPB
    in their native distributions (Section~\ref{sec:local-bpb}).
    Values are percentages of retained text bytes.
    Figure~\ref{fig:bpb-models} shows the full BPB distributions.}
  \label{tab:bpb-summary}
  \small
  \begin{tabular}{lcccc}
    \toprule
    Model & $F_X(10^{-3})$ & $F_X(0.1)$ & $S_X(4)$ & $S_X(8)$ \\
    \midrule
Ours \byte{}-1B & 23.3 & 56.6 & 7.43 & 0.822 \\
BLT-1B & 25.7 & 59.2 & 6.62 & 0.730 \\
H-Net 1-stage XL & 22.5 & 55.8 & 7.53 & 0.878 \\
H-Net 2-stage XL & 23.1 & 56.2 & 7.48 & 0.867 \\
\midrule
Ours \subword{}-1B & \phantom{0}1.5 & 16.3 & 2.35 & 0.218 \\
Llama-3.1-8B & \phantom{0}6.5 & 27.2 & 1.18 & 0.090 \\
Qwen3-8B & 10.5 & 29.3 & 1.91 & 0.204 \\
Qwen3.5-9B & \phantom{0}5.3 & 24.9 & 1.44 & 0.133 \\
 
    \bottomrule
  \end{tabular}
\end{table}

\paragraph{Token boundaries explain the distribution gap.}
We aggregate \byte{} BPB by averaging within each token span defined
by a subword tokenizer. The aggregated
distribution (dotted lines in Figure~\ref{fig:information}b--c) closely
matches the native \subword{} distribution.\footnote{We also ablated
boundaries from the Qwen3 and Qwen3.5 tokenizers,
which yielded similar aggregated distributions
(Figure~\ref{fig:bpb-aggregation}; Appendix~\ref{app:bpb-aggregation}).}
This result indicates that the distribution gap largely reflects how
information is grouped into tokens.
Subword tokenization provides a statistical prior for this grouping:
it combines high-BPB bytes with predictable continuations into fewer
tokens, reducing near-zero BPB mass while retaining the source's total
information (Proposition~\ref{prop:recoding}).

\paragraph{Byte modeling can allocate more computation to difficult tokens.}
For the same total source information, \byte{} uses multiple modeling
steps within a span that \subword{} handles in one step, retaining
high-BPB bytes as separate tokens (Figure~\ref{fig:information}c).
This supports the argument
in Section~\ref{sec:parameter-compute} that longer sequences can scale
computation without increasing model size.

\subsection{Drafting Low-BPB Continuations}
\label{sec:drafting}

Maintaining a fixed byte granularity allows more computation for
high-BPB tokens, but also adds modeling steps for the many low-BPB
continuations (Figure~\ref{fig:information}; Table~\ref{tab:bpb-summary}).
We therefore examine whether speculative decoding
\citep{leviathan2023speculative} can reduce this overhead.

We use existing 200M and 50M models as drafts and same-family 1B models
as targets for both \byte{} and \subword{}. We do not train specialized
drafts. On the validation dataset, we compare accepted bytes per target
forward pass in an oracle simulation on the given evaluation text.
Appendix~\ref{app:drafting} specifies the protocol.

\paragraph{Higher byte acceptance rates compensate for shorter token spans.}
At both draft sizes, \byte{} accepts about $3.4\times$ as many tokens
per target forward pass as \subword{}. Accepted bytes per forward are
nearly identical with 200M drafts and slightly higher for \byte{} with
50M drafts (Table~\ref{tab:draft}). The 200M drafts yield longer accepted
text per forward in both families.

\begin{table}[htbp]
  \centering
  \caption{Higher byte acceptance rates offset shorter token spans at both
    draft sizes, giving comparable accepted text per target forward pass.
    Targets are same-family 1B models, evaluated by an oracle
    simulation on the given evaluation text
    (Appendix~\ref{app:draft-settings}). Forward denotes a simulated
    target verification pass; counts include only accepted draft tokens
    (Appendix~\ref{app:draft-measurement}).}
  \label{tab:draft}
  \begin{tabular}{lcccc}
    \toprule
    & \multicolumn{2}{c}{200M draft model} & \multicolumn{2}{c}{50M draft model} \\
    \cmidrule(lr){2-3}\cmidrule(lr){4-5}
    & \byte{} & \subword{} & \byte{} & \subword{} \\
    \midrule
    Acceptance rate (\%) & 90.6 & 74.1 & 84.6 & 62.0 \\
    Accepted tokens / forward & 9.527 & 2.831 & 5.471 & 1.618 \\[-2pt]
    {\small\textit{Byte / Subword ratio}} & \multicolumn{2}{c}{\small $3.366\times$} & \multicolumn{2}{c}{\small $3.381\times$} \\
    Bytes / accepted token & 1.000 & 3.364 & 1.000 & 3.159 \\
    \addlinespace
    \textbf{Accepted bytes / forward} & \textbf{9.527} & \textbf{9.523} & \textbf{5.471} & \textbf{5.111} \\
    \bottomrule
  \end{tabular}
\end{table}

These acceptance lengths suggest that many additional byte tokens
can be verified together, reducing the need for separate sequential target
calls. The comparable text coverage of the two families at both draft
sizes suggests that joint verification could offset much of the extra
sequential overhead introduced by byte granularity.

\section{Related Work}
\label{sec:related}

\paragraph{Tokenizer-free language models.}
Tokenizer-free language models operate directly on fine-grained inputs such
as bytes or characters, avoiding a separately constructed subword vocabulary.
ByT5 \citep{xue2022byt5} studies byte-to-byte pretraining, while MambaByte
\citep{wang2024mambabyte} models byte sequences with selective state-space
models. 
Other works \citep{clark2022canine, yu2023megabyte, pagnoni2025blt,hwang2025hnet}, 
introduce hierarchical mechanisms that compress
or aggregate fine-grained inputs for coarser-grained computation.
In contrast, we study byte language modeling with 
Transformers with hybrid
attention that preserves byte resolution throughout the network.

\paragraph{Compute scaling and allocation.}
Model scale and computation can be varied independently through tokenization
and architectural design. Compute Optimal Tokenization
\citep{limisiewicz2026tokenization} studies how token granularity changes
the relationship between raw data, sequence length, and compute.
Repeated-depth models such as LoopFormer \citep{jeddi2026loopformer}
increase computation while sharing parameters, whereas sparse MoE models
such as DeepSeek-V3 \citep{deepseekv3} increase parameter capacity without
proportionally increasing activated computation.
Building on these directions, we use shared depth to increase \subword{}
computation and sparse experts to reduce \byte{} activated computation at
a given parameter capacity, examining how computation shapes the
byte--subword modeling trade-off.

\section{Discussion and Conclusion}
\label{sec:discussion}



Our results provide a complementary perspective to recent byte-level approaches, which introduce explicit hierarchical or learned tokenization mechanisms to control the cost of long byte sequences. Using a flat Transformer, we show that this sequence-length overhead can instead provide useful computation: byte models benefit from scaling and favor substantially more training data per parameter under a fixed FLOP budget. Moreover, although we do not impose a hierarchy, byte Transformers develop segmentation-like positions that aggregate local text information, and intermediate layers can exploit these representations without loss in downstream performance. This connects flat byte Transformers to hierarchical architectures and suggests that local abstraction and global modeling can emerge implicitly from byte-level training. These learned structures also have inference-time consequences, as prediction difficulty concentrates near their boundaries and can be exploited by speculative decoding. Together, our results suggest that sequence length, learned abstraction, and computation should be viewed as coupled design dimensions, motivating future work on adaptive computation and learned computational granularity for byte language models.

\paragraph{Limitations}
Our study is limited to vanilla Transformer backbones and the evaluated training and task scales; whether the observed scaling behavior and emergent local structures generalize to other architectures, larger models, and broader modalities remains open.

\label{maintextend}

\clearpage
\section*{Statement on the Use of Generative AI}
In this work, we used generative AI tools to assist with experimental design, method implementation and debugging, training and mathematical exposition.
We also used these tools for literature retrieval, figure and table
preparation, and manuscript drafting and editing.
We take full responsibility for the final manuscript and associated artifacts,
including AI-assisted text, code, analyses, and claims.

\section*{Reproducibility Statement}
We provide model configurations, training hyperparameters, and evaluation
settings for the main model suite and IsoFLOPs sweep in
Appendices~\ref{app:architecture}--\ref{app:scaling}.
Settings for BPB distribution analyses and experiments on draft
acceptance and local information structure are detailed in
Appendices~\ref{app:lgl}--\ref{app:drafting}.

\section*{Ethics Statement}
To the best of our knowledge, this work raises no specific ethical concerns.

\bibliography{iclr2027_conference}
\bibliographystyle{iclr2027_conference}

\clearpage
\appendix
\section{Architecture and Objectives}
\label{app:architecture}

\subsection{Tokenizers}
Our byte tokenizer encodes text directly as UTF-8 without learned merges
or Unicode normalization. To support document/chat control and padding
within a 256-entry vocabulary, we use \texttt{0xC0}, \texttt{0xC1}, and
\texttt{0xF5}--\texttt{0xFF} as special-token IDs.
These 13 bytes values are forbidden in UTF-8 and never occur in valid UTF-8
text \citep{yergeau2003utf8}.

For subword models, we use \texttt{HuggingFaceTB/cosmo2-tokenizer}
\citep{benallal2025smollm2}\footnote{\url{https://huggingface.co/HuggingFaceTB/cosmo2-tokenizer}}, a pretrained byte-level BPE tokenizer
with 49,152 vocabulary entries.
On the pretraining dataset, it averages approximately 4.56 UTF-8 bytes per token.
We use this ratio in the compute and data exposure calculations in
Section~\ref{sec:parameter-compute} and Appendices~\ref{app:isoflops}.

\subsection{Token-Superposition Training (TST)}
\label{app:TST}

TST \citep{peng2026tst} lets each model input represent several tokens.
We divide a token sequence $\mathbf T$ into non-overlapping bags of $K=4$ consecutive tokens, forming the sequence
$\mathbf B=(B_1,\ldots,B_{|\mathbf T|/K})$.
Let $T_{i,j}$ be the $j$th token in bag $B_i$.
Each bag is represented by the average of its token embeddings,
including causal hash features when enabled.

The causal model and its single output head, with parameters $\theta$,
use the preceding bags $\mathbf B_{<i}$ to produce one shared token
distribution $p_\theta(\cdot\mid\mathbf B_{<i})$ for all tokens in
the next bag $B_i$.
The loss averages token cross-entropy over all target bags
\citep[Eq.~(3)]{peng2026tst}:
\begin{equation}
  \mathcal L_{\mathrm{TST}}(\mathbf B;\theta)
  =-\frac{1}{|\mathbf T|-K}
    \sum_{i=2}^{|\mathbf T|/K}\sum_{j=1}^{K}
      \ln p_\theta(T_{i,j}\mid\mathbf B_{<i}).
  \label{eq:tst}
\end{equation}

The first 30\% of pretraining updates use TST; the remaining
70\% use ordinary next-token training on the same parameters.
During TST, each four-token bag counts as one training token.
With a fixed number of training tokens per update, this
schedule provides $0.3\times4+0.7\times1=1.9$ times the text exposure of
ordinary next-token training at the same training-token count.

\subsection{Hashed Embeddings}
\label{app:hash_emb}
To supply local context at the input, we augment token
embeddings with hashed $n$-gram representations inspired by BLT
\citep{pagnoni2025blt}.
For a sequence $x_{1:T}$ of byte or subword tokens, let
$\mathcal G$ denote the set of $n$-gram groups, indexed by their
lengths. The input representation is
\begin{equation}
  e_i=\frac{1}{1+|\mathcal G|}\left[
    E(x_i)+\sum_{n\in\mathcal G}E_n\!\left(\operatorname{Hash}(x_{i-n+1:i})\right)
  \right].
  \label{eq:hash}
\end{equation}
Here $E$ is the ordinary token embedding table, and $E_n$ is a separate
embedding table for length $n$, indexed by the shared polynomial hash $\operatorname{Hash}$.
When $i<n$, the missing left context is padded with zeros.

For \byte{}, $\mathcal G=\{5,6,7,8\}$, with four hash tables of
32,768 slots each. For \subword{}, $\mathcal G=\{2\}$,
with one hash table of 131,072 slots.
Both families therefore use 131,072 hash slots in total.

Table~\ref{tab:architecture-differences} lists the settings that differ
between \byte{} and \subword{}.

\begin{table}[htbp]
  \centering
  \caption{Architectural differences between \byte{} and \subword{}.
    TST and hash settings apply when the corresponding components are enabled.}
  \label{tab:architecture-differences}
  \small
  \begin{tabular}{lcc}
    \toprule
    Hyperparameter & \byte{} & \subword{} \\
    \midrule
    Vocabulary size & 256 & 49,152 \\
    Convolution kernel width & 16 & 4 \\
    TST bag size / duration & \multicolumn{2}{c}{4 / first 30\% of updates} \\
    Hash $n$-gram lengths $\mathcal G$ & $\{5,6,7,8\}$ & $\{2\}$ \\
    Hash tables $\times$ slots per table & $4\times32{,}768$ & $1\times131{,}072$ \\
    \bottomrule
  \end{tabular}
\end{table}

\subsection{Dense model configuration}
\label{app:model-configurations}
We scale the Qwen3.5 text architecture \citep{qwen35blog} using the
dense configurations in Table~\ref{tab:config}.
Its hybrid attention combines Gated DeltaNet (GDN)
and Gated Softmax Attention \citep{qiu2025gated} in a 3:1 ratio.
We increase depth in four-layer increments to preserve this ratio.
All models are trained from scratch with untied input and output embeddings.

We use $N$ for the number of unique active model parameters, including
the output embedding.
Input/hash embeddings are excluded.
The difference in $N$ between \subword{} and \byte{} models
comes mainly from the larger subword output embedding, with a smaller
contribution from their convolution widths.

\begin{table}[htbp]
  \centering
  \caption{Dense configurations and active parameter counts $N$
    (billions).
    In the head-count columns, kv indicates equal key and value counts;
    qk indicates equal query and key counts.}
  \label{tab:config}
  \small
  \setlength{\tabcolsep}{3pt}
  \begin{tabular}{ccccccccc}
    \toprule
    Size & Layers & Width $d$ & FFN & \begin{tabular}[c]{@{}c@{}}Head dim.\\GQA/GDN\end{tabular} &
    \begin{tabular}[c]{@{}c@{}}GQA\\q/kv heads\end{tabular} &
    \begin{tabular}[c]{@{}c@{}}GDN\\qk/v heads\end{tabular} &
    \byte{} $N$ & \subword{} $N$ \\
    \midrule
200M & 16 & 1,024 & 2,560 & 128/64 & 8/2 & 8/16 & 0.207 & 0.256 \\
400M & 20 & 1,280 & 3,200 & 128/64 & 10/2 & 10/20 & 0.374 & 0.436 \\
700M & 24 & 1,536 & 3,840 & 128/64 & 12/2 & 12/24 & 0.644 & 0.719 \\
1B & 28 & 1,792 & 4,480 & 128/64 & 14/2 & 14/28 & 1.022 & 1.109 \\
3B & 40 & 2,560 & 6,400 & 256/128 & 10/2 & 10/20 & 3.181 & 3.305 \\
 
    \bottomrule
  \end{tabular}
\end{table}

\subsection{MoE configuration}
\label{app:moe}
Hidden widths and attention configurations follow Table~\ref{tab:config}.
All layers use MoE, and the expert widths in Table~\ref{tab:experts}
apply to both routed and shared experts.
Each layer uses 8 active experts selected from 64 routed experts,
together with 1 shared expert.
For MoE models, $N$ denotes the active parameter count: it includes
routers, selected routed experts, and shared experts, and excludes
inactive routed experts.
$N_{\mathrm{total}}$ counts all model parameters except hash embeddings.

\begin{table}[htbp]
  \centering
  \caption{MoE configurations and parameter counts (billions).}
  \label{tab:experts}
  \small
  \setlength{\tabcolsep}{3pt}
  \begin{tabular}{ccccccc}
    \toprule
    Size & Layers & Expert width & \byte{} $N$ & \byte{} $N_{\mathrm{total}}$ &
    \subword{} $N$ & \subword{} $N_{\mathrm{total}}$ \\
    \midrule
200M & 16 & 256 & 0.180 & \phantom{0}0.885 & 0.230 & \phantom{0}0.985 \\
400M & 20 & 384 & 0.395 & \phantom{0}2.047 & 0.457 & \phantom{0}2.172 \\
700M & 24 & 384 & 0.604 & \phantom{0}2.983 & 0.679 & \phantom{0}3.132 \\
1B & 28 & 512 & 1.044 & \phantom{0}5.361 & 1.131 & \phantom{0}5.535 \\
3B & 40 & 768 & 3.146 & 16.358 & 3.269 & 16.607 \\
 
    \bottomrule
  \end{tabular}
\end{table}

We use the auxiliary-loss-free load balancing strategy of
DeepSeek-V3 \citep{deepseekv3}.
For a hidden state $h$, the router computes token-to-expert affinities
$s_e=\sigma(w_e^\top h)$, where $w_e$ is the router vector for expert $e$.
Each routed expert has a bias $b_e$, initialized to zero.
The eight experts with the largest $s_e+b_e$ form the selected set $\mathcal K(h)$.
The selected experts are combined with weights
\begin{equation}
  \alpha_e=
    \frac{2.5s_e}{\sum_{j\in\mathcal K(h)}s_j},
  \qquad
  \operatorname{MoE}(h)=f_{\mathrm{shared}}(h)
       +\sum_{e\in\mathcal K(h)}\alpha_e f_e(h).
  \label{eq:moe}
\end{equation}
The bias affects expert selection, while the gating weights are computed
from the original affinities $s_e$.
After each training step, we update the biases according to
\begin{equation}
  b_e\leftarrow b_e+\gamma\operatorname{sign}(\bar n-n_e),
  \label{eq:moe-balance}
\end{equation}
where $n_e$ is the number of token assignments to expert $e$ across the
whole training batch, $\bar n$ is the mean count over routed experts,
and $\gamma$ is the bias update speed.
We use $\gamma=10^{-3}$ and freeze the biases during the final 5\%
of training updates.

\subsection{Looped model configuration}
\label{app:loop}
LoopFormer \citep{jeddi2026loopformer} repeatedly applies a shared
Transformer stack, conditioning each pass on normalized time and step size.
It combines next-token prediction on full and shorter trajectories with
a shortcut-consistency objective to support elastic-depth inference.

Our looped \subword{} control adopts the shared-stack recurrence with
four fixed passes, retaining the baseline decoder blocks and ordinary
next-token objective.
Let $H^{(0)}$ denote the input token embeddings and $\Phi_\theta$ the
complete causal decoder stack with parameters $\theta$. The recurrence is
\begin{equation}
  H^{(r)}=\Phi_\theta\!\left(H^{(r-1)}\right),
  \qquad r=1,\ldots,4.
  \label{eq:loop}
\end{equation}
Both training and evaluation use this recurrence, with a single
prediction head after the final pass.
Training compute includes all four decoder passes and one output-head
application, with forward and backward matrix operations counted as
six FLOPs per participating weight per processed token.
$N$ counts each shared parameter once.

\section{Main Model Suite: Training Details}
\label{app:training}

This appendix describes training for the main model suite: the
200M--3B \byte{} and \subword{} models evaluated on CUTE, QA RC, and
vision--language benchmarks in Section~\ref{sec:finegrained}, together with the
looped and MoE models compared in Section~\ref{sec:allocation}.

\subsection{Pretraining hyperparameters}
\label{app:pretraining}
All model families are pretrained on ClimbMix \citep{diao2025climb}.
The 16-, 20-, 24-, 28-, and 40-layer models use 50k, 60k, 70k, 80k, and
100k updates, respectively.
The byte sequence length is four times the subword sequence length;
with the batch sizes in Table~\ref{tab:pretraining}, the two families
process approximately the same amount of source text per update during
ordinary next-token training. Because byte and subword models also use
the same number of updates at each scale, their total source-text
exposure under this training regime is approximately matched.

\begin{table}[htbp]
  \centering
  \caption{Default pretraining hyperparameters for the main model suite.}
  \label{tab:pretraining}
  \small
  \begin{tabularx}{0.70\linewidth}{l*{2}{>{\centering\arraybackslash}X}}
    \toprule
    Hyperparameter & \byte{} & \subword{} \\
    \midrule
    Context length (tokens) & 8,192 & 2,048 \\
    Global batch (tokens) & 1,048,576 & 262,144 \\
    \midrule
    Base learning rate $\eta_0$ & \multicolumn{2}{c}{0.005} \\
    Optimizers & \multicolumn{2}{c}{Muon (matrices), AdamW (other parameters)} \\
    Scheduler & \multicolumn{2}{c}{WSD, 1\% warmup, 20\% 1-sqrt decay} \\
    AdamW $(\beta_1,\beta_2)$ & \multicolumn{2}{c}{$(0.8,0.95)$} \\
    AdamW $\epsilon$ & \multicolumn{2}{c}{$10^{-10}$} \\
    AdamW weight decay & \multicolumn{2}{c}{0} \\
    Muon $\beta_1$ & \multicolumn{2}{c}{$0.85\xrightarrow{\text{400 steps}}0.97\xrightarrow{\text{final 20\%}}0.90$} \\
    Muon $\beta_2$ & \multicolumn{2}{c}{0.95} \\
    Muon iterations & \multicolumn{2}{c}{5} \\
    Computation precision & \multicolumn{2}{c}{bfloat16} \\
    \bottomrule
  \end{tabularx}
\end{table}

We use $\mu$P-style hyperparameter transfer \citep{yang2022mup}, with the
following width and batch corrections in our Muon/AdamW optimizer
\citep{jordan2024muon,loshchilov2019adamw}.
Let $\eta_0$ be the common base learning rate in Table~\ref{tab:pretraining}
and $d_0=768$ the reference width.
Let $r$ and $c$ denote the numbers of rows and columns, respectively,
of a parameter matrix updated by Muon. The peak learning rates for
AdamW and Muon are
\begin{equation}
  \begin{aligned}
    \eta_{\mathrm{Adam}}&=\eta_0\sqrt{2}\sqrt{d_0/d},\\
    \eta_{\mathrm{Muon}}(r,c)&=\eta_0\sqrt{2}\sqrt{\max(1,r/c)}.
  \end{aligned}
  \label{eq:mup-scaling}
\end{equation}
The factor $\sqrt{2}$ comes from learning-rate scaling with batch size.
We neither tune nor automatically scale the global batch size, so this factor remains constant across
model scales.
Our pretraining hyperparameter settings are largely based on
nanochat\footnote{\url{https://github.com/karpathy/nanochat}}.

\subsection{Text supervised fine-tuning}
\label{app:sft}
We evaluate CUTE (Appendix~\ref{app:cute}) after SFT for both byte and subword
models. Language-modeling and QA RC evaluations use base models.
The SFT mixture combines Smol-SmolTalk \citep{benallal2025smollm2}, MMLU, GSM8K \citep{cobbe2021gsm8k},
and CUTE-style tasks, with loss on assistant responses.
SFT starts at 0.8 of the pretraining learning rate, with 10\%
warmup steps and decay to zero over the final 50\% of updates.

\subsection{Vision--language training details}
\label{app:vision-training}

\paragraph{Architecture.}
We adapt \byte{} and \subword{} models into vision--language models (VLMs) using
an identical vision--language architecture. A frozen SigLIP2 SO400M vision encoder connects to the
pretrained language model through a two-layer GELU MLP projector,
which maps visual representations into the language embedding space.
The projected visual tokens enter the language sequence through a unified
image-token interface.

\paragraph{Training.}
We use a two-stage training pipeline following the standard LLaVA recipe.
Stage~1 trains only the projector on LLaVA-Pretrain, keeping the vision
encoder and language model frozen, with learning rate $10^{-3}$ and
effective batch size 256.
Stage~2 jointly fine-tunes the projector and language model on
LLaVA-NeXT-SFT, keeping the vision encoder frozen, with learning rate
$2\times10^{-5}$ and effective batch size 128.
Both stages use AdamW, bfloat16 mixed precision, 3\% warmup followed by cosine
learning-rate decay, and one training epoch.

\section{IsoFLOPs Sweep Details}
\label{app:isoflops}

\subsection{Sweep settings}
The IsoFLOPs sweep builds on the pretraining setup in
Appendix~\ref{app:pretraining} and uses the BPB evaluation protocol in
Appendix~\ref{app:evaluation-data}. We evaluate dense byte and subword
models under four recipes: raw, hash-only, TST-only, and TST+hash.
To cover the optimal model sizes at low FLOP budgets, we extend the
dense configurations in Table~\ref{tab:config} down to 8 layers.
For each model size, the number of training updates is determined by
the FLOP budget, replacing the prescribed update counts of the main
model suite. Warmup lasts 1,000 updates instead of 1\% of training;
the remaining training hyperparameters, including decay over the final
20\% of updates, follow Table~\ref{tab:pretraining}. Architecture and
enabled TST/hash components use the settings in
Table~\ref{tab:architecture-differences}.
Figure~\ref{fig:isoflops} presents the complete IsoFLOPs sweep results
across all evaluated model sizes and FLOP budgets.

\begin{figure}[htbp]
  \centering
  \includegraphics[width=\linewidth]{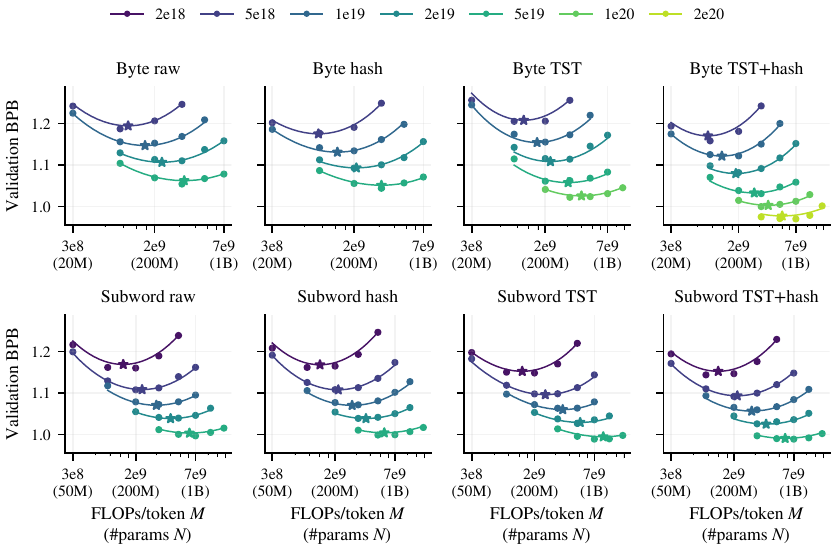}
  \caption{TST+hash lowers fitted optimal BPB in both families at shared
    budgets, with larger gains for bytes. Dots are measurements;
    curves are fixed-budget slices of $L(M,V)$, and stars
    are fitted optima. Colors denote training FLOP budgets.
    }
  \label{fig:isoflops}
\end{figure}

\subsection{Parameter, FLOP, and data accounting}
\paragraph{Training FLOPs per token.}
To compare compute costs across model sizes and token granularities,
we estimate training FLOPs per token as
\begin{equation}
  M=6N
    +\sum_{\ell\in\mathcal A}12H^Q_\ell d^h_\ell T
    +\sum_{\ell\in\mathcal D}24H^V_\ell d^k_\ell d^v_\ell.
  \label{eq:isoflops-token-compute}
\end{equation}
The sums add full-attention and Gated DeltaNet recurrent operations,
with layer sets $\mathcal A$ and $\mathcal D$, respectively.
$T$ is the context length in training tokens, $H^Q_\ell,H^V_\ell$ are
query/value head counts, and $d^h_\ell,d^k_\ell,d^v_\ell$ their dimensions.
By convention, embedding lookups, hashing, bag averaging, loss evaluation,
and optimizer updates are excluded from this estimate.

To convert fitted loss curves and compute-optimal allocations from $M$
to $N$ smoothly between trained sizes, we express both counts in terms
of model width $d$. Substituting the architecture hyperparameters and
context lengths of the discrete sweep configurations gives
\begin{equation}
  \begin{aligned}
    N_{\mathrm{byte}}(d)&=\frac{1443}{8192}d^3+\frac{19715}{8192}d^2+\frac{1031}{4}d,\\
    N_{\mathrm{subword}}(d)&=\frac{1443}{8192}d^3+\frac{17411}{8192}d^2+\frac{196615}{4}d,\\
    M_{\mathrm{byte}}(d)&=\frac{4329}{4096}d^3+\frac{1665}{4}d^2+1536d,\\
    M_{\mathrm{subword}}(d)&=\frac{4329}{4096}d^3+\frac{2025}{16}d^2+294912d.
  \end{aligned}
  \label{eq:isoflops-architecture}
\end{equation}
We then allow $d$ to vary continuously in these expressions to obtain
the mapping $N(M)$.

\paragraph{Source-byte exposure volume.}
Let $D$ be the number of training tokens, counting each superposed bag
as one token during TST. Cumulative source-byte exposure is $V=sD$,
where $s$ is the mean number of source bytes per training token,
averaged over the training schedule.
Ordinary next-token training uses $s=1$ for bytes and $s=4.561464$
for subwords. TST with four-token bags for the first 30\% of updates
multiplies either value by 1.9 (Appendix~\ref{app:TST}). It increases
source-byte exposure at a fixed training-token count while leaving
per-token compute $M$ unchanged.

\paragraph{Total training FLOPs.}
Combining per-token compute and source-byte exposure gives
\begin{equation}
  C_{\mathrm{iso}}=MD=\frac{MV}{s}.
  \label{eq:isoflops-compute}
\end{equation}
At a fixed FLOP budget, this relation determines the source-byte
exposure available to each model size.

\subsection{Global fit and compute-optimal allocation}
We fit Equation~\ref{eq:mv-law} independently to all observations of each
recipe, adapting the Chinchilla law
\citep[Section 3.3 and Appendix D.2]{hoffmann2022computeoptimal}
to per-token compute $M$ and source bytes $V$.
We estimate the parameters by minimizing a Huber loss with threshold
0.001 between predicted and observed log losses. Bounded L-BFGS-B uses
500 random initializations.

\begin{figure}[htbp]
  \centering
  \includegraphics[width=\linewidth]{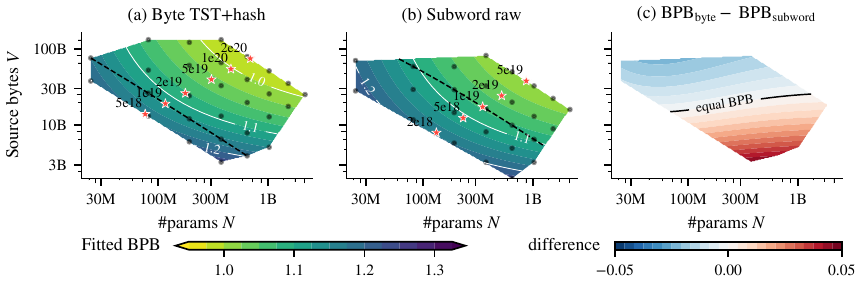}
  \caption{At matched $N,V$, \byte{} TST+hash improves over raw
    \subword{} as source-byte exposure grows. Here, $N$ excludes the output
    embedding to align parameter counts between the two model families.
    Panels (a--b) show fitted BPB, and (c) shows their difference.
    Stars are fitted compute-optimal allocations;
    black dots are training configurations, and dashed contours mark
    $10^{19}$ FLOPs. Surfaces cover only the sampled parameter--data regions.}
  \label{fig:isoflops-surfaces}
\end{figure}

The interior optimum at budget $C$ is
\begin{equation}
  M^\star(C)=
  \left[\frac{\alpha A}{\beta B}(sC)^\beta\right]^{1/(\alpha+\beta)},
  \qquad
  V^\star(C)=\frac{sC}{M^\star(C)}.
  \label{eq:mv-optimum}
\end{equation}
Let $N^\star(C)$ be the parameter count corresponding to $M^\star(C)$
and $L^\star(C)$ the resulting fitted loss. The interior balance
$\alpha A(M^\star)^{-\alpha}=\beta B(V^\star)^{-\beta}$ yields
\begin{equation}
  \frac{V^\star}{N^\star}
  =\left(\frac{\beta B}{\alpha A}\right)^{1/\beta}
    \frac{[M(N^\star)]^{\alpha/\beta}}{N^\star},
  \qquad
  L^\star=E+\left(1+\frac{\alpha}{\beta}\right)A[M(N^\star)]^{-\alpha}.
  \label{eq:mv-frontier}
\end{equation}
Thus the optimal data-to-parameter ratio depends jointly on the learned
model/data exponents and architectural compute costs.
For the reported allocations, we clip $M^\star$ to each recipe's observed
size range and convert it to $N^\star$ using Equation~\ref{eq:isoflops-architecture}.
Varying $C$ yields the $(N^\star(C),L^\star(C))$ and
$(C,L^\star(C))$ Pareto frontiers in Figure~\ref{fig:isoflops-frontiers}.

\begin{table}[htbp]
  \centering
  \caption{Fitted parameters of $L(M,V)=E+AM^{-\alpha}+BV^{-\beta}$
    (Section~\ref{sec:parameter-compute}).}
  \label{tab:isoflops-fits}
  \begin{tabular}{lccccc}
    \toprule
    Recipe & $A$ & $\alpha$ & $B$ & $\beta$ & $E$ \\
    \midrule
    Byte raw & 1054.9 & 0.42293 & 13257 & 0.50811 & 0.87358 \\
Byte hash & 300.19 & 0.36138 & 53129 & 0.57329 & 0.86467 \\
Byte TST & 2547.6 & 0.46305 & 3691.2 & 0.42706 & 0.83219 \\
Byte TST+hash & 550.58 & 0.39097 & 9343.3 & 0.46963 & 0.81537 \\
Subword raw & 2410.6 & 0.46131 & 10460 & 0.48918 & 0.85900 \\
Subword hash & 3810 & 0.48746 & 10338 & 0.48703 & 0.86609 \\
Subword TST & 229.07 & 0.33740 & 81150 & 0.57364 & 0.83360 \\
Subword TST+hash & 1825.6 & 0.44984 & 14020 & 0.48763 & 0.84645 \\
 
    \bottomrule
  \end{tabular}
\end{table}

\section{Evaluation Protocols and Additional Results}
\label{app:scaling}

\subsection{Validation bits per byte}
\label{app:evaluation-data}
We evaluate language modeling on the small held-out perplexity validation
set used in OLMo~3 training \citep{teamolmo2025olmo3}\footnote{\url{https://github.com/allenai/OLMo-core/blob/main/src/olmo_core/data/mixes/v3-small-ppl-validation.txt}}.
It contains 18,523 documents from C4, Dolma (books,
Common Crawl, peS2o, Reddit, The Stack, and Wikipedia), ICE, M2D2 S2ORC,
the Pile, and WikiText-103.

To fairly compare across tokenizations, we normalize
next-token cross-entropy by the number of text bytes.
Let $\mathbf T=(T_1,\ldots,T_{|\mathbf T|})$ be the token sequence
and $\mathbf B=(B_1,\ldots,B_{|\mathbf B|})$
the UTF-8 byte of the same sequence. For a model with
distribution $p_\theta$, BPB is calculated as
\begin{align}
  \mathcal L_{\mathrm{CE}}(\mathbf T;\theta)
    &=-\frac{1}{|\mathbf T|}\sum_{i=1}^{|\mathbf T|}
      \ln p_\theta(T_i\mid\mathbf T_{<i}),
  \nonumber\\
  \bpb
    &=\frac{|\mathbf T|}{|\mathbf B|}
      \frac{\mathcal L_{\mathrm{CE}}(\mathbf T;\theta)}{\ln 2}.
  \label{eq:global-bpb}
\end{align}

\subsection{Question answering (QA RC)}
\label{app:qarc}
We evaluate base-model question answering (QA) using the QA RC task group
of the OLMo~3 Base-Easy suite \citep{teamolmo2025olmo3}, implemented in
OLMES \citep{gu2025olmes}\footnote{\url{https://github.com/allenai/olmes}}.
RC denotes ranked classification. Each candidate is scored separately by
summing the conditional token log-probabilities of its associated
continuation, using the suite's task-specific completion prompts and
in-context examples. 

We report QA RC as the macro-average of 21 task-family scores, expressed
as a percentage. For ARC, MMLU, and Basic Skills, constituent task
accuracies are first averaged with equal weight within each family;
the final average then gives each task equal weight.
Table~\ref{tab:qarc-main-details} breaks down QA RC by task; Figure~\ref{fig:family-scores}
extends the QA RC comparison to all model families.

\begin{table}[htbp]
  \centering
  \caption{QA RC scores (\%; $\uparrow$) for our models.}
  \label{tab:qarc-main-details}
  \scriptsize
  \setlength{\tabcolsep}{2pt}
  \begin{tabular}{l*{10}{c}}
    \toprule
    & \multicolumn{2}{c}{200M} & \multicolumn{2}{c}{400M}
    & \multicolumn{2}{c}{700M} & \multicolumn{2}{c}{1B} & \multicolumn{2}{c}{3B} \\
    \cmidrule(lr){2-3}\cmidrule(lr){4-5}\cmidrule(lr){6-7}\cmidrule(lr){8-9}\cmidrule(lr){10-11}
    Task & \subword{} & \byte{} & \subword{} & \byte{} & \subword{} & \byte{} & \subword{} & \byte{} & \subword{} & \byte{} \\
    \midrule
    ARC & 54.27 & 53.67 & 55.56 & 58.58 & 59.84 & 60.49 & 61.40 & 62.83 & 65.81 & 68.36 \\
MMLU & 29.68 & 29.57 & 31.09 & 31.50 & 32.83 & 33.63 & 34.02 & 35.64 & 37.08 & 38.63 \\
CSQA & 53.24 & 52.50 & 57.08 & 54.71 & 57.08 & 60.11 & 62.24 & 61.34 & 64.62 & 69.12 \\
HellaSwag & 47.24 & 51.32 & 52.48 & 56.91 & 57.14 & 62.17 & 60.94 & 65.56 & 67.84 & 72.89 \\
WinoGrande & 51.78 & 52.72 & 53.83 & 55.17 & 55.88 & 59.19 & 58.01 & 58.96 & 60.14 & 62.98 \\
SocialIQA & 43.24 & 46.11 & 47.08 & 47.54 & 48.87 & 48.26 & 48.87 & 51.02 & 52.76 & 54.76 \\
PiQA & 70.57 & 70.02 & 72.74 & 73.61 & 74.37 & 75.19 & 75.08 & 76.99 & 78.29 & 78.84 \\
CoQA & 67.48 & 68.53 & 71.29 & 72.46 & 73.42 & 75.49 & 74.79 & 76.82 & 81.72 & 82.79 \\
DROP & 27.30 & 28.77 & 28.66 & 29.53 & 31.09 & 31.81 & 30.50 & 32.59 & 34.26 & 37.86 \\
Jeopardy & 37.30 & 36.78 & 44.12 & 40.69 & 46.16 & 48.02 & 53.12 & 52.83 & 61.60 & 64.46 \\
NaturalQs & 24.81 & 24.72 & 25.69 & 27.00 & 27.39 & 29.18 & 29.82 & 31.61 & 33.31 & 35.68 \\
SQuAD & 66.80 & 71.26 & 71.90 & 76.77 & 76.87 & 80.64 & 78.94 & 84.08 & 85.76 & 90.49 \\
SciQ & 78.10 & 78.90 & 82.40 & 84.10 & 84.90 & 86.30 & 86.80 & 89.60 & 90.10 & 91.40 \\
QASPER & 67.08 & 63.32 & 67.08 & 61.76 & 66.14 & 65.52 & 67.08 & 62.38 & 65.52 & 63.32 \\
Basic Skills & 49.30 & 51.66 & 51.89 & 53.47 & 54.97 & 57.79 & 58.15 & 59.40 & 64.51 & 65.08 \\
DBQA & 27.12 & 28.65 & 25.38 & 30.19 & 29.42 & 29.04 & 28.85 & 27.88 & 27.88 & 25.58 \\
ProtocolQA & 27.78 & 26.85 & 27.78 & 25.93 & 24.07 & 24.07 & 25.93 & 25.93 & 26.85 & 26.85 \\
LAMBADA & 36.97 & 39.06 & 40.21 & 42.07 & 43.94 & 45.45 & 45.47 & 48.26 & 51.31 & 53.70 \\
MedMCQA & 27.11 & 28.50 & 27.99 & 29.17 & 28.38 & 30.00 & 29.14 & 30.60 & 30.96 & 32.42 \\
MedQA & 21.45 & 22.70 & 23.33 & 23.10 & 24.12 & 24.74 & 25.14 & 25.14 & 27.42 & 30.32 \\
SciRIFF & 72.63 & 72.69 & 74.21 & 74.40 & 76.49 & 79.01 & 74.15 & 80.21 & 85.15 & 84.77 \\
\midrule
Mean & 46.73 & 47.54 & 49.13 & 49.94 & 51.11 & 52.67 & 52.78 & 54.27 & 56.80 & 58.59 \\
 
    \bottomrule
  \end{tabular}
\end{table}

\subsection{CUTE: character and word manipulation}
\label{app:cute}
CUTE \citep{edman2024cute} tests spelling, inverse spelling, containment,
and character/word insertion, deletion,
swapping and substitution, including random-string variants.
We evaluate 19 variants using fixed demonstration examples and greedy
generation. 
Containment scores are adjusted from $a$ to $\max(0,2a-1)$ to remove the
binary chance baseline. CUTE is evaluated with models after text SFT (Appendix~\ref{app:sft}).

Aggregate and task-level scores for main models are given in Tables~\ref{tab:cute-scaling}
and~\ref{tab:cute-main-details}, respectively; Figure~\ref{fig:family-scores}
extends the CUTE comparison to all model families.

\begin{table}[htbp]
  \centering
  \caption{CUTE scores (\%; $\uparrow$) for our main models.}
  \label{tab:cute-main-details}
  \scriptsize
  \setlength{\tabcolsep}{2pt}
  \begin{tabular}{l*{10}{c}}
    \toprule
    & \multicolumn{2}{c}{200M} & \multicolumn{2}{c}{400M}
    & \multicolumn{2}{c}{700M} & \multicolumn{2}{c}{1B} & \multicolumn{2}{c}{3B} \\
    \cmidrule(lr){2-3}\cmidrule(lr){4-5}\cmidrule(lr){6-7}\cmidrule(lr){8-9}\cmidrule(lr){10-11}
    Task & \subword{} & \byte{} & \subword{} & \byte{} & \subword{} & \byte{} & \subword{} & \byte{} & \subword{} & \byte{} \\
    \midrule
    Spelling & 51.70 & 99.80 & 56.10 & 99.80 & 60.00 & 99.90 & 68.60 & 99.80 & 81.80 & 100.00 \\
Spelling (random) & 99.10 & 99.90 & 99.00 & 99.90 & 99.40 & 99.00 & 99.60 & 100.00 & 99.80 & 100.00 \\
Inverse spelling & 59.90 & 100.00 & 63.20 & 100.00 & 73.90 & 100.00 & 78.30 & 100.00 & 82.00 & 100.00 \\
Inverse spelling (random) & 87.30 & 100.00 & 93.00 & 99.90 & 91.20 & 100.00 & 95.60 & 100.00 & 97.60 & 100.00 \\
Character containment & 67.20 & 99.20 & 70.60 & 99.60 & 72.60 & 100.00 & 74.00 & 99.80 & 78.60 & 100.00 \\
Containment (random) & 99.00 & 100.00 & 100.00 & 100.00 & 100.00 & 100.00 & 100.00 & 100.00 & 100.00 & 100.00 \\
Word containment & 99.80 & 99.00 & 100.00 & 98.80 & 100.00 & 99.60 & 99.80 & 99.80 & 100.00 & 100.00 \\
Character insertion & 17.30 & 99.70 & 19.80 & 99.30 & 21.20 & 99.30 & 23.20 & 99.20 & 31.50 & 99.80 \\
Insertion (random) & 79.40 & 99.90 & 80.90 & 99.90 & 84.60 & 100.00 & 92.20 & 100.00 & 95.10 & 100.00 \\
Word insertion & 97.60 & 98.10 & 97.90 & 97.50 & 97.30 & 98.70 & 98.30 & 98.60 & 98.70 & 99.80 \\
Character deletion & 24.10 & 99.40 & 24.70 & 99.60 & 30.50 & 99.20 & 32.20 & 99.50 & 41.10 & 100.00 \\
Deletion (random) & 95.40 & 100.00 & 95.20 & 100.00 & 96.40 & 99.90 & 98.20 & 100.00 & 99.00 & 100.00 \\
Word deletion & 96.70 & 97.20 & 97.80 & 94.30 & 97.80 & 98.70 & 98.40 & 98.70 & 98.10 & 99.70 \\
Character swapping & 4.00 & 99.50 & 3.20 & 99.50 & 6.00 & 99.70 & 9.60 & 99.60 & 18.00 & 100.00 \\
Swapping (random) & 59.30 & 99.90 & 51.50 & 99.90 & 57.80 & 100.00 & 73.10 & 100.00 & 89.50 & 100.00 \\
Word swapping & 94.20 & 95.50 & 95.90 & 98.00 & 96.50 & 97.90 & 97.50 & 98.00 & 98.80 & 99.50 \\
Character substitution & 17.20 & 99.50 & 19.70 & 99.80 & 22.50 & 99.60 & 25.30 & 99.50 & 32.00 & 100.00 \\
Substitution (random) & 80.60 & 100.00 & 78.40 & 100.00 & 81.90 & 100.00 & 89.70 & 99.90 & 94.80 & 100.00 \\
Word substitution & 98.10 & 98.10 & 97.90 & 98.60 & 98.10 & 99.40 & 98.70 & 99.30 & 99.40 & 99.80 \\
\midrule
Mean & 69.89 & 99.19 & 70.78 & 99.18 & 73.04 & 99.52 & 76.44 & 99.56 & 80.83 & 99.93 \\
 
    \bottomrule
  \end{tabular}
\end{table}

\paragraph{Prompts and task examples.}
Each evaluation prompt presents the task instruction, four numbered
question--answer demonstrations, and the test question in a single user
message. The demonstrations are fixed for each task and shared by
\byte{} and \subword{}; random-string variants reuse the corresponding
task's instruction and demonstrations.
Table~\ref{tab:cute-examples} shows a complete spelling prompt and selected
evaluation questions with their reference answers.

\begin{table}[htbp]
  \centering
  \caption{CUTE evaluation prompt and task examples. The spelling prompt is
    shown in full; the additional examples show only the final question.
    Reference answers are displayed separately from the model input.}
  \label{tab:cute-examples}
  \small
  \begin{minipage}{\linewidth}
    \textbf{Complete spelling prompt (user message)}\par\smallskip
    {\ttfamily\raggedright
    Spell out the word, putting spaces between each letter, based on the following examples:\par\smallskip
    1. Question: Spell out the word "alphabet". Answer: "a l p h a b e t"\par
    2. Question: Spell out the word "hello". Answer: "h e l l o"\par
    3. Question: Spell out the word "zebra". Answer: "z e b r a"\par
    4. Question: Spell out the word "tongue". Answer: "t o n g u e"\par\smallskip
    Question: Spell out the word "there". Respond only with the final answer. Answer:\par}
    \smallskip
    \textbf{Reference answer:} \texttt{t h e r e}
  \end{minipage}
  \medskip

  \setlength{\tabcolsep}{4pt}
  \renewcommand{\arraystretch}{1.15}
  \begin{tabularx}{\linewidth}{@{}>{\raggedright\arraybackslash}p{0.20\linewidth}
    >{\ttfamily\raggedright\arraybackslash}X
    >{\ttfamily\raggedright\arraybackslash}p{0.16\linewidth}@{}}
    \toprule
    Task & \textrm{Final question} & \textrm{Reference answer} \\
    \midrule
    Spelling (random) & Question: Spell out the word "bdj". & b d j \\
    Inverse spelling & Question: Write the word "t h e". & the \\
    Character insertion & Question: Add an "l" after every "t" in "little". & litltlle \\
    Word swapping & Question: Swap "is" and "fun" in "It is fun.". & It fun is. \\
    \bottomrule
  \end{tabularx}
\end{table}

\subsection{Vision--language evaluation}
\label{app:vision}
We evaluate with the \texttt{lmms-eval} framework~\citep{zhang2025lmmseval}
on vision--language benchmarks covering several capabilities.
DocVQA~\citep{mathew2021docvqa}, OCRBench, TextVQA, and
ChartQA~\citep{masry2022chartqa} assess OCR and document understanding;
ScienceQA-IMG~\citep{lu2022scienceqa} and
MMBench~\citep{liu2023mmbench} assess knowledge-intensive visual reasoning.
MME~\citep{fu2023mme}, GQA~\citep{hudson2019gqa}, and
POPE~\citep{li2023pope} assess general
visual understanding, covering object hallucination, visual question
answering, and perception and cognition.

\paragraph{OCR task examples.}
The OCRBench examples in Table~\ref{tab:ocr-examples} illustrate direct
transcription of irregular, handwritten, and non-semantic text.
The TextVQA examples require identifying the text relevant to a question
within a scene.

\begin{table}[htbp]
  \centering
  \caption{OCRBench and TextVQA examples with questions and reference answers.
    Images are omitted.}
  \label{tab:ocr-examples}
  \small
  \setlength{\tabcolsep}{4pt}
  \renewcommand{\arraystretch}{1.15}
  \renewcommand{\tabularxcolumn}[1]{m{#1}}
  \begin{tabularx}{\linewidth}{@{}>{\raggedright\arraybackslash}m{0.23\linewidth}
    >{\raggedright\arraybackslash}X
    >{\ttfamily\raggedright\arraybackslash}m{0.20\linewidth}@{}}
    \toprule
    Benchmark / task & Question & \textrm{Reference answer} \\
    \midrule
    OCRBench\newline Irregular text
    & What is written in the image?
    & PARLIAMENT \\
    \addlinespace[3pt]
    OCRBench\newline Handwriting
    & What is written in the image?
    & strictures \\
    \addlinespace[3pt]
    OCRBench\newline Non-semantic text
    & What is written in the image?
    & ntishgcwi \\
    \midrule
    TextVQA\newline Scene text QA
    & What does the small white text spell?
    & copenhagen \\
    \addlinespace[3pt]
    TextVQA\newline Scene text QA
    & What is the license plate number of this vehicle?
    & aj52uyv \\
    \addlinespace[3pt]
    TextVQA\newline Scene text QA
    & What is the phone number listed to rent this billboard?
    & 648-3004 \\
    \bottomrule
  \end{tabularx}
\end{table}

Table~\ref{tab:vision-all} and Table~\ref{tab:vision-moe} show that performance
on the remaining vision--language tasks is more task-dependent than on
OCRBench and TextVQA. Across both dense and MoE models, the differences
between \byte{} and \subword{} are mixed at smaller scales, but generally
become more favorable to \byte{} as model scale or capacity increases.
For the largest dense models, \byte{} surpasses \subword{} on
ScienceQA-IMG, MMBench, and POPE and nearly matches it on GQA, while the
largest \byte{} MoE model achieves higher scores on several broader
vision--language benchmarks, including ScienceQA-IMG, MMBench, MME, and
GQA. Together with the consistent gains on OCRBench and TextVQA reported in
Table~\ref{tab:cute-scaling}, these results suggest that the clearest
advantage of byte-level modeling appears on fine-grained vision--language tasks,
while broader vision--language capabilities become increasingly competitive
with scale and capacity.

\begin{table*}[!t]
  \centering
  \caption{Vision--language evaluation for dense models.}
  \label{tab:vision-all}

  \fontsize{8}{10}\selectfont
  \setlength{\tabcolsep}{2.5pt}
  \begin{tabular}{@{}cl*{9}{c}@{}}
    \toprule
    & & \multicolumn{4}{c}{\shortstack{OCR \& document\\understanding}}
    & \multicolumn{2}{c}{\shortstack{Knowledge-intensive\\visual reasoning}}
    & \multicolumn{3}{c}{\shortstack{General visual\\understanding}} \\
    \cmidrule(lr){3-6}\cmidrule(lr){7-8}\cmidrule(lr){9-11}
    Size & Model
    & DocVQA
    & OCRBench
    & TextVQA
    & ChartQA
    & ScienceQA-IMG
    & MMBench
    & MME
    & GQA
    & POPE \\
    \midrule
200M & \subword{}
     & 25.98 & 27.10 & 40.62 & 24.04
     & \textbf{42.49} & \textbf{\phantom{0}8.16}
     & 1164.20 & \textbf{48.26} & \textbf{85.37} \\
200M & \byte{}
     & \textbf{28.56} & \textbf{35.10} & \textbf{45.08} & \textbf{26.56}
     & 37.88 & \phantom{0}2.06
     & \textbf{1187.69} & 46.26 & 84.65 \\
\midrule

400M & \subword{}
     & 29.77 & 27.40 & 43.82 & 28.00
     & 46.01 & \textbf{23.11}
     & \textbf{1245.94} & \textbf{49.67} & 82.28 \\
400M & \byte{}
     & \textbf{32.64} & \textbf{37.40} & \textbf{48.37} & \textbf{28.72}
     & \textbf{48.19} & 13.48
     & 1206.10 & 49.13 & \textbf{82.95} \\
\midrule

700M & \subword{}
     & 31.35 & 29.90 & 45.34 & 30.08
     & \textbf{50.17} & 29.98
     & \textbf{1332.91} & \textbf{51.87} & \textbf{85.39} \\
700M & \byte{}
     & \textbf{33.43} & \textbf{34.50} & \textbf{50.72} & \textbf{30.56}
     & 47.65 & \textbf{30.15}
     & 1326.96 & 50.50 & 82.46 \\
\midrule

1B & \subword{}
     & \textbf{34.38} & 31.50 & 47.50 & 30.96
     & \textbf{54.93} & \textbf{36.68}
     & 1327.14 & \textbf{53.21} & \textbf{84.99} \\
1B & \byte{}
     & 34.19 & \textbf{39.40} & \textbf{52.36} & \textbf{32.12}
     & 47.89 & 31.01
     & \textbf{1344.80} & 52.28 & 84.51 \\
\midrule

3B & \subword{}
     & 37.40 & 34.70 & 50.76 & 36.92
     & 56.97 & 44.15
     & \textbf{1470.32} & \textbf{56.38} & 85.72 \\
3B & \byte{}
     & \textbf{38.83} & \textbf{42.40} & \textbf{55.39} & \textbf{38.36}
     & \textbf{60.54} & \textbf{45.62}
     & 1428.47 & 55.99 & \textbf{86.13} \\
 
    \bottomrule
  \end{tabular}
\end{table*}

\begin{table*}[!t]
  \centering
  \caption{Vision--language evaluation for MoE models.}
  \label{tab:vision-moe}

  \fontsize{8}{10}\selectfont
  \setlength{\tabcolsep}{2.5pt}
  \begin{tabular}{@{}cl*{9}{c}@{}}
    \toprule
    & & \multicolumn{4}{c}{\shortstack{OCR \& document\\understanding}}
    & \multicolumn{2}{c}{\shortstack{Knowledge-intensive\\visual reasoning}}
    & \multicolumn{3}{c}{\shortstack{General visual\\understanding}} \\
    \cmidrule(lr){3-6}\cmidrule(lr){7-8}\cmidrule(lr){9-11}
    Size & Model
    & DocVQA
    & OCRBench
    & TextVQA
    & ChartQA
    & ScienceQA-IMG
    & MMBench
    & MME
    & GQA
    & POPE \\
    \midrule
200M & \subword{}
     & \textbf{25.87} & 26.00 & 39.79 & 23.04
     & \textbf{50.37} & \textbf{\phantom{0}2.66}
     & \textbf{1219.18} & \textbf{46.10} & \textbf{82.26} \\
200M & \byte{}
     & 25.85 & \textbf{32.90} & \textbf{43.22} & \textbf{23.28}
     & 39.96 & \phantom{0}1.37
     & 1154.95 & 43.81 & 81.36 \\
\midrule

400M & \subword{}
     & 28.68 & 27.10 & 44.01 & 23.48
     & 44.97 & \phantom{0}6.35
     & 1265.12 & 47.06 & \textbf{84.78} \\
400M & \byte{}
     & \textbf{30.84} & \textbf{35.90} & \textbf{49.61} & \textbf{27.32}
     & \textbf{47.30} & \textbf{18.21}
     & \textbf{1292.77} & \textbf{47.61} & 84.75 \\
\midrule

700M & \subword{}
     & 34.26 & 33.40 & 49.28 & 34.48
     & \textbf{52.85} & \textbf{35.57}
     & 1335.96 & 53.65 & 85.17 \\
700M & \byte{}
     & \textbf{35.15} & \textbf{39.60} & \textbf{53.50} & \textbf{35.36}
     & 51.36 & 30.41
     & \textbf{1341.47} & \textbf{53.78} & \textbf{86.04} \\
\midrule

1B & \subword{}
     & \textbf{36.39} & \textbf{33.40} & \textbf{50.13} & \textbf{35.88}
     & \textbf{56.17} & 30.32
     & 1247.24 & \textbf{54.91} & 86.90 \\
1B & \byte{}
     & 33.36 & 32.30 & 42.45 & 34.72
     & 54.04 & \textbf{38.57}
     & \textbf{1390.83} & 49.39 & \textbf{86.96} \\
\midrule

3B & \subword{}
     & \textbf{42.17} & 39.70 & 53.57 & 40.84
     & 58.50 & 46.13
     & 1469.19 & 57.87 & \textbf{87.36} \\
3B & \byte{}
     & 41.87 & \textbf{46.10} & \textbf{59.70} & \textbf{43.08}
     & \textbf{61.47} & \textbf{55.41}
     & \textbf{1576.49} & \textbf{59.06} & 87.33 \\
 
    \bottomrule
  \end{tabular}
\end{table*}

\clearpage
\subsection{Parameter and compute efficiency across model families}

\begin{figure}[htbp]
  \centering
  \includegraphics[width=\linewidth]{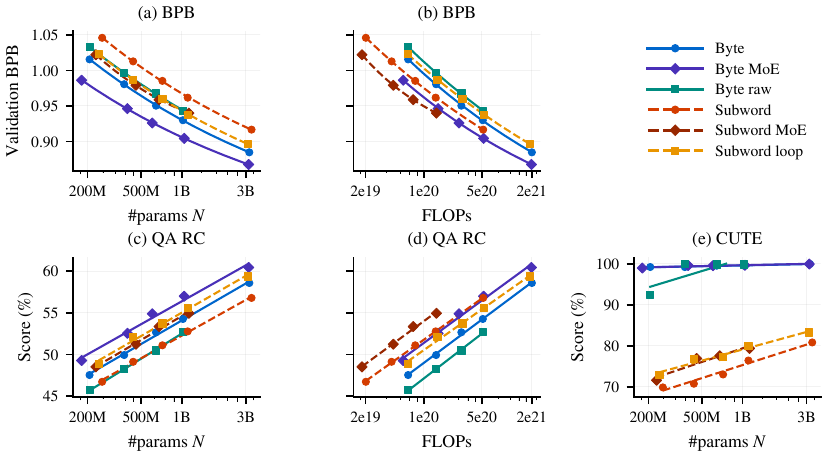}
  \caption{In the main model suite, dense \byte{} achieves lower BPB and
    higher QA RC than depth-matched dense \subword{}, with greater
    training compute.
    BPB and QA RC use base models, and CUTE uses SFT models.
    Points are observations; curves summarize trends within each
    family's observed range (Appendix~\ref{app:scaling}).}
  \label{fig:family-scores}
\end{figure}

\section{Details of Local Information Structure Experiments}
\label{app:lgl}

This appendix provides additional details for the local information structure
experiments in Section~\ref{sec:lgl}. We first describe the input-copying task,
including the prompt construction, input variants, and attention measurement
procedure. Additional analyses with alternative separator settings are also
provided. We then describe the frozen layer intervention, including local span
construction, compressed global stage design, and downstream evaluation
settings.

\subsection{Input-copying Task Setting}
\label{app:abstraction-copy}

\paragraph{Prompt and input variants.}
Copying requires the model to retrieve information from an earlier occurrence
of the same text, making it a useful probe for studying how information is
organized and accessed internally. Following the use of repeated sequences to
study induction heads \citep{olsson2022induction}, we inspect
\byte{}-1B with the prompt
\begin{center}
  \small\texttt{Your task is copy the string: X now copy it: X}
\end{center}
Both occurrences of $X$ are supplied in a single teacher-forced causal
forward pass. The string contains ten six-word English phrases; each word has
exactly 5 ASCII bytes, allowing word boundaries to align across examples.
For example, one phrase is \texttt{"quiet birds cross ocean waves daily"}.
The original and space-free variants contain 35 and 30 bytes, respectively.

\paragraph{Attention measurement.}
For copied byte offset $r$, the query corresponds to the position immediately
before that byte, whose hidden state predicts the next byte. The key set
$\mathcal{K}$ contains the source phrase and its preceding separator position,
denoted as $B$. Let
$A_{n,h}^{(\ell)}(r,k)$ denote the raw attention weight from this query to
source key $k$ in example $n$, head $h$, and layer $\ell$.
With $N=10$ examples and $H=14$ heads, we compute

\begin{equation}
  \bar A^{(\ell)}_{r,k}
    =
    \frac{1}{NH}
    \sum_{n=1}^{N}\sum_{h=1}^{H}
    A_{n,h}^{(\ell)}(r,k),
\end{equation}

and normalize each row within the source region:

\begin{equation}
  \widetilde A^{(\ell)}_{r,k}
  =
  \frac{\bar A^{(\ell)}_{r,k}}
  {\sum_{j\in\mathcal{K}}\bar A^{(\ell)}_{r,j}} .
\end{equation}

Thus, each row represents the attention distribution from the query
position for a copied byte token to the source sequence. Word-end guides
indicate the final byte of each original word; in the space-free condition, these guides are only
annotations and do not appear in the input.

\paragraph{Alternative separator analysis.}
To investigate how word separators affect attention, we compare spaces with
\texttt{|}, \texttt{;}, and \texttt{and}. Only separators inside $X$ change;
the underlying phrases and surrounding instruction remain fixed.
Figure~\ref{fig:implicit-copy-separators} shows the resulting attention maps.

\begin{figure}[!htb]
  \centering
  \includegraphics[width=\linewidth]{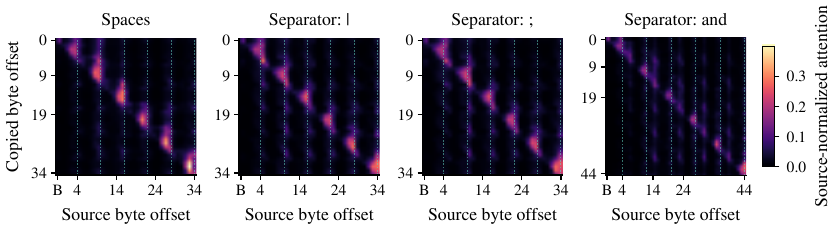}
  \caption{\textbf{Ablations on separators: Local retrieval persists with alternative separators.}
    Layer-24 \byte{}-1B maps use the copying probe in Figure~\ref{fig:implicit-copy},
    with words separated by spaces, \texttt{|}, \texttt{;}, or \texttt{and}.
    Dashed lines mark original word ends; colors share the same
    source-normalized attention scale.}
  \label{fig:implicit-copy-separators}
\end{figure}

\subsection{Frozen Layer Intervention}
\label{app:abstraction-lgl}

\begin{figure}[!htb]
  \centering
  \includegraphics[width=\linewidth]{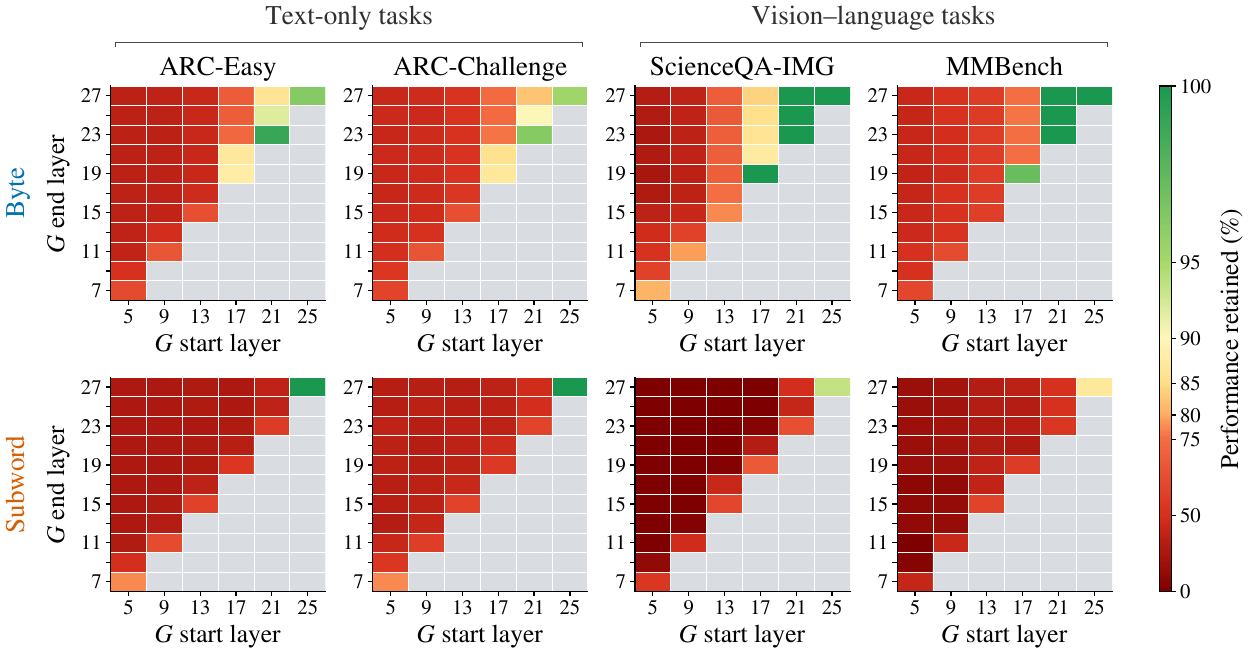}
  \caption{\textbf{Additional performance retention results under the frozen layer intervention.}
Each cell reports downstream performance relative to the uncompressed
baseline for a compressed layer interval \(G=[a,b]\).
The axes denote the start layer \(a\), end layer \(b\), and interval
length \(|G|=b-a+1\).
\byte{} results are shown in the top row and \subword{} results in the
bottom row. 100\% indicates performance at or above the uncompressed
baseline.}
  \label{fig:implicit-lgl-sweep-additional}
\end{figure}

\paragraph{Local span construction.}
The frozen layer intervention introduces a compressed global stage over local
representations while preserving the original byte-resolution computation
before and after the selected layers. Inspired by the local--global--local
organization of BLT and H-Net \citep{pagnoni2025blt,hwang2025hnet}, we identify
local spans in the prompts of the 28-layer \byte{}-1B and \subword{}-1B models
using a matched 200M model from the corresponding family.
For a prompt token $T_i$, the smaller model with parameters $\theta$
computes the token loss

\begin{equation}
  \mathcal L_{\mathrm{CE}}(T_i\mid\mathbf T_{<i};\theta)
    =-\ln p_\theta(T_i\mid\mathbf T_{<i}).
  \label{eq:token-ce}
\end{equation}

Tokens with high loss mark patch boundaries, grouping predictable
continuations into local spans.

For text-only tasks with $T$ prompt tokens, we target
$\lceil T/4\rceil$ patches, corresponding to approximately $4\times$
compression of the text sequence. The first token always starts a patch,
and we select the highest loss threshold that yields the target number of
patch starts, retaining all ties.

For vision--language tasks with $T$ non-image prompt tokens and $V$ image
tokens, we apply the same $4\times$ target ratio only to the text portion,
giving a target of
$V+\lceil T/4\rceil$ patches. Each image token is kept as a singleton patch,
while text patch boundaries are selected using the same loss threshold
procedure. In both settings, prompt patches remain fixed during generation,
and each newly generated token is appended as a singleton patch.

\paragraph{Compressed global stage.}
Let $G=[a,b]$ denote the selected layer interval. Before layer $a$, the model
operates at the original resolution and produces hidden states $h_i$. For a
local span $P_j$, we construct a pooled representation

\begin{equation}
z_j=\frac{1}{|P_j|}\sum_{i\in P_j}h_i .
\end{equation}

The original layers in $G$ then operate on the shorter sequence of local
representations. After processing, the global updates are mapped back to the
original resolution through a residual reconstruction:

\begin{equation}
\hat h_i = h_i+\Delta_{j(i)},
\qquad
j(i)=\max\{j:e_j\le i\}.
\end{equation}

Here, $e_j$ is the final position of span $P_j$, and $\Delta_j$ is its
representation update after the global stage. Tokens before the first
completed span keep their original states.

This preserves the original byte-level states while requiring the selected
layers to process only local aggregated representations.

\paragraph{Evaluation settings and additional results.}

We sweep the start layer
$a\in\{5,9,13,17,21,25\}$ and interval lengths
$|G|\in\{3,5,\ldots,23\}$, retaining intervals with $b\le27$.
The final layer operates at full resolution after reconstruction.
No parameter updates or additional training are performed.

For text-only evaluation, we use CUTE, ARC-Easy,
ARC-Challenge, and MMLU. For vision--language evaluation, we use
OCRBench, TextVQA, GQA, ScienceQA-IMG, and MMBench.
Results on CUTE, MMLU, OCRBench, TextVQA, and GQA are reported in
Figure~\ref{fig:implicit-lgl-sweep}.
The remaining results on ARC-Easy, ARC-Challenge, ScienceQA-IMG, and MMBench are shown in
Figure~\ref{fig:implicit-lgl-sweep-additional}.

\section{BPB Distribution Analyses}
\label{app:derivation}

\subsection{Source entropy and token granularity}
\label{app:bpb-theory}

We first establish a common total entropy for different tokenizations
of the same text source. Here $H$ denotes Shannon entropy in bits.
\begin{proposition}[Entropy invariance across tokenizers]
\label{prop:recoding}
Let $\mathbf B$ be a random complete finite byte sequence with finite
entropy. For two fixed deterministic tokenizers $\tau_1$ and $\tau_2$,
define $\mathbf T^{(j)}=\tau_j(\mathbf B)$ for $j\in\{1,2\}$.
If each encoding uniquely determines the original byte sequence, then
\begin{equation}
  H(\mathbf T^{(1)})=H(\mathbf T^{(2)})=H(\mathbf B).
  \label{eq:text-entropy}
\end{equation}
\end{proposition}
\begin{proof}
Let $\mathcal S$ be the support of $\mathbf B$. For either tokenizer
$\tau_j$, unique recovery implies that
$\tau_j(\mathbf b)=\tau_j(\mathbf b')$ only if $\mathbf b=\mathbf b'$.
Thus $\tau_j$ is a bijection between $\mathcal S$ and the support
$\tau_j(\mathcal S)$ of $\mathbf T^{(j)}$. Let $p_B$ be the source
probability mass function and $p_j$ the induced probability mass
function of $\mathbf T^{(j)}$. We have
$p_j(\tau_j(\mathbf b))=p_B(\mathbf b)$ for every $\mathbf b\in\mathcal S$,
so reindexing the entropy sum gives
\begin{equation*}
  \begin{aligned}
    H(\mathbf T^{(j)})
    &=-\sum_{\mathbf b\in\mathcal S}p_j(\tau_j(\mathbf b))
      \log_2 p_j(\tau_j(\mathbf b))\\
    &=-\sum_{\mathbf b\in\mathcal S}p_B(\mathbf b)
      \log_2 p_B(\mathbf b)=H(\mathbf B).
  \end{aligned}
\end{equation*}
\end{proof}
For each complete text $\mathbf b$ in the source support, the same
probability correspondence preserves its total information
$-\log_2 p_B(\mathbf b)$ across tokenizations.

Byte prediction exposes the intermediate prefixes inside a subword
span, whereas a subword model predicts the span jointly. This source
invariance motivates the comparison, but model BPB also reflects
approximation error and available context.

\subsection{Experimental settings for BPB distribution analysis}
\label{app:bpb-protocol}

We compare models on a 2 MiB subset of the evaluation dataset in
Appendix~\ref{app:evaluation-data}. Subword models are
evaluated with sequences of up to 4,096 tokens.
Figure~\ref{fig:bpb-models} shows the full BPB distributions.

\paragraph{H-Net target filtering.}
\label{app:bpb-hnet}
H-Net's highest BPB values concentrate on TAB (\texttt{0x09}),
LF (\texttt{0x0A}), CR (\texttt{0x0D}), \texttt{\{} (\texttt{0x7B}),
and \texttt{\}} (\texttt{0x7D}). H-Net were trained
on FineWeb-Edu, a filtered subset of FineWeb.
FineWeb's C4 filter rejects an entire document when a surviving line
contains \texttt{\{}, also removing any paired \texttt{\}} in that
document. It drops empty lines and joins retained lines with a single
LF, while text extraction collapses whitespace such as TAB and CR.
These preprocessing choices may explain the extreme BPB values through
reduced training exposure to braces, TAB, CR, and repeated line breaks.

To limit the influence of these formatting differences on the distribution
comparison, we exclude the five target characters from H-Net's distribution
statistics while retaining them in the context.
H-Net thus retains 98.5\% of the full evaluation set.

\begin{figure}[!htb]
  \centering
  \includegraphics[width=\linewidth]{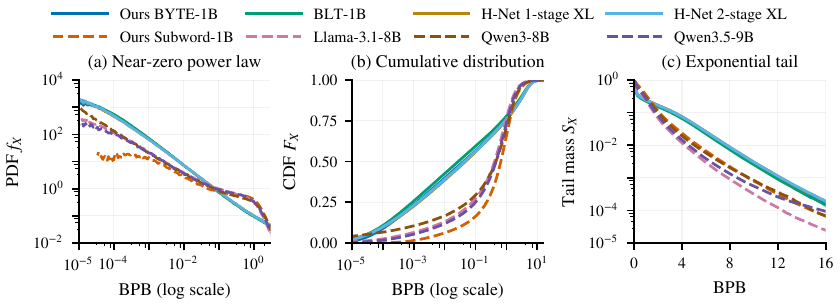}
  \caption{Byte models show more near-zero
    and high-BPB mass than subword models. All panels are byte-weighted.}
  \label{fig:bpb-models}
\end{figure}

\paragraph{Finite-range density fits.}
To identify the BPB ranges where each density law holds, we use the
local power-law exponent and exponential decay scale shown in
Figure~\ref{fig:bpb-fits}:
\begin{equation}
  \alpha_{\rm local}(x)=\frac{d\ln f_X(x)}{d\ln x},
  \qquad
  \beta_{\rm local}(x)=-\left[\frac{d\ln f_X(x)}{dx}\right]^{-1}.
  \label{eq:local-density-parameters}
\end{equation}

For the power law, we select the widest stable interval spanning at
least 1.5 decades; for the exponential tail, we select the longest
stable core spanning at least 4 BPB. Both require global and local
log-density RMS errors at most 0.15, and local power-law exponents
must also lie within 0.15 of the fitted $\alpha$.

Table~\ref{tab:density} reports the resulting fitting ranges,
$\alpha$, and $\beta$.

\begin{table}[htbp]
  \centering
  \caption{Finite-range density fits $f_X\propto x^\alpha$ near zero
    and $f_X\propto e^{-x/\beta}$ in the tail, with $\beta$ in BPB.
    Byte exponents are nearer $-1$; tail scales are comparable across
    granularities.}
  \label{tab:density}
  \begin{tabular}{lcccc}
    \toprule
    Model & Near-zero BPB & $\alpha$ & Tail BPB & $\beta$ \\
    \midrule
Ours \byte{}-1B & $[5.4\times10^{-5},2.15]$ & -0.946 & $[3,15.5]$ & 1.898 \\
BLT-1B & $[5.4\times10^{-5},2.86]$ & -0.952 & $[3.5,15.5]$ & 1.855 \\
H-Net 1-stage XL & $[4.1\times10^{-5},5.07]$ & -0.916 & $[5.5,16]$ & 1.867 \\
H-Net 2-stage XL & $[3.1\times10^{-5},5.07]$ & -0.918 & $[5.5,16.5]$ & 1.899 \\
Ours \subword{}-1B & $[4.0\times10^{-4},0.164]$ & -0.612 & $[3.5,13]$ & 1.684 \\
Llama-3.1-8B & $[1.3\times10^{-5},0.164]$ & -0.696 & $[2.5,11.25]$ & 1.407 \\
Qwen3-8B & $[9.9\times10^{-7},0.0923]$ & -0.789 & $[3,13.5]$ & 1.735 \\
Qwen3.5-9B & $[3.1\times10^{-5},0.164]$ & -0.691 & $[2.5,11.75]$ & 1.492 \\
 
    \bottomrule
  \end{tabular}
\end{table}

\begin{figure}[p]
  \centering
  \includegraphics[width=\linewidth]{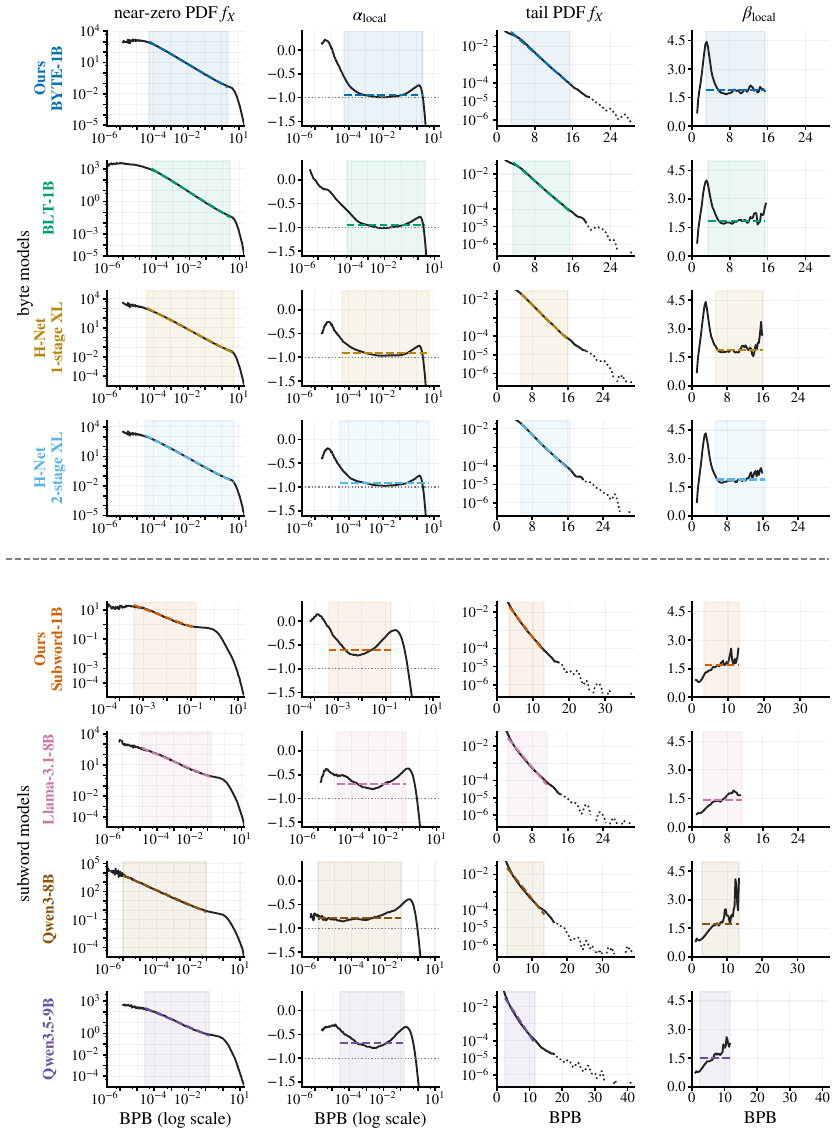}
  \caption{Near-zero power laws and exponential tail cores describe
    finite BPB ranges across models. All panels use byte-weighted distributions.
    Dashed lines show fits, and shading
    marks fitting ranges. Local exponents $\alpha_{\rm local}$ and decay
    scales $\beta_{\rm local}$ indicate where the approximations hold.
    Dotted tail segments indicate sparse data. Fitting details are in
    Appendix~\ref{app:bpb-protocol}.}
  \label{fig:bpb-fits}
\end{figure}

\subsection{BPB aggregation ablations}
\label{app:bpb-aggregation}

We test the sensitivity of BPB aggregation to tokenizer choice using
\byte{}-1B and BLT-1B. For each model, we average BPB over complete spans
defined by the \subword{}, Qwen3, and Qwen3.5 tokenizers and assign
the span mean to each constituent byte. The aggregated distributions
are similar across the three tokenizers (Figure~\ref{fig:bpb-aggregation}).

\paragraph{Uniform four-byte grouping.}
To test whether average span length alone explains the aggregated
distribution, we also use non-overlapping four-byte groups, close to \subword{}'s average
token length in bytes. For both byte models, uniform grouping
reduces variation but leaves more near-zero mass and less far-tail mass
than tokenizer grouping (Figure~\ref{fig:bpb-aggregation}). These
differences highlight the role of tokenizer boundaries in shaping the
aggregated BPB distribution.

\paragraph{Tokenizers share most boundaries.}
On the evaluation dataset, we find high pairwise boundary Jaccard
similarity among the \subword{}, Llama, Qwen3, and Qwen3.5 tokenizers
(Figure~\ref{fig:tokenizer-agreement}). These similar boundaries group
largely the same bytes together, helping explain the consistency of
BPB aggregation across tokenizers.

\begin{figure}[!htb]
  \centering
  \includegraphics[width=\linewidth]{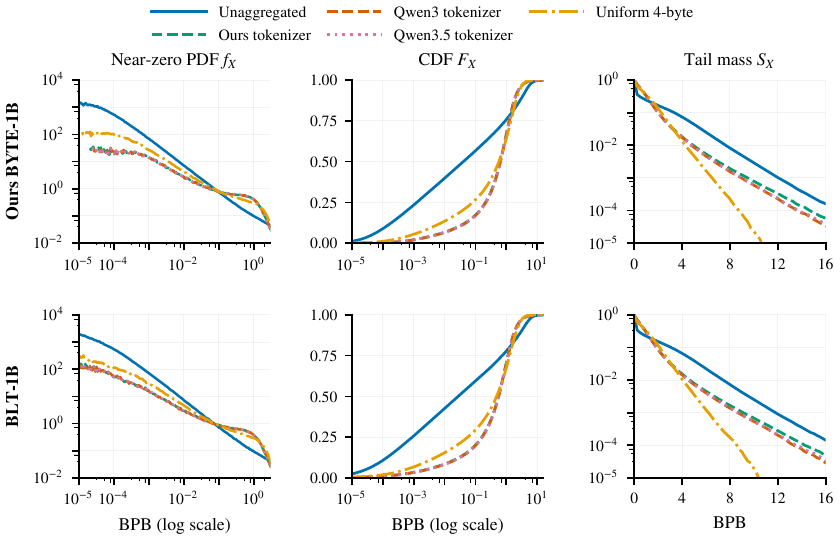}
  \caption{Token-span averaging reduces near-zero and high-BPB mass
    for both \byte{}-1B and BLT-1B. All panels are byte-weighted.
    Tokenizers yield similar
    distributions; uniform 4-byte grouping leaves more near-zero mass
    and less far-tail mass (Appendix~\ref{app:bpb-aggregation}).}
  \label{fig:bpb-aggregation}
\end{figure}

\begin{figure}[!htb]
  \centering
  \includegraphics[width=.44\linewidth]{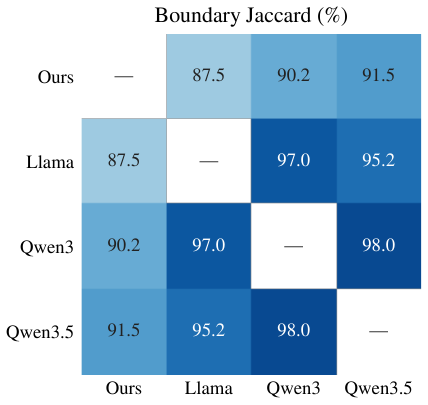}
  \caption{Tokenizers share most boundaries, consistent with their
    similar aggregation curves. Values are boundary Jaccard
    similarities (\%).}
  \label{fig:tokenizer-agreement}
\end{figure}

\newpage
\section{Draft Acceptance and Accepted Text}
\label{app:drafting}

\subsection{Models and evaluation settings}
\label{app:draft-settings}

The 200M, 16-layer drafts and the 1B targets come from the main model
suite, with the dense configurations in Table~\ref{tab:config} and
training settings in Appendix~\ref{app:training}. The 50M drafts use
8-layer models from the IsoFLOPs sweep.

We evaluate documents from the validation dataset in
Appendix~\ref{app:evaluation-data} within each family's native context
window, 8,192 byte tokens or 2,048 subword tokens, respectively.

\subsection{Oracle simulation and acceptance metrics}
\label{app:draft-measurement}

To isolate draft--target agreement from the strategy used to construct
draft candidates, we use an oracle simulation on the given evaluation
text. It directly measures consecutive agreements, setting aside
choices such as draft-tree expansion, pruning, and stopping heuristics.

Both models condition on the preceding
tokens of this text at each scored token (teacher forcing).
A draft token is counted as accepted when the two models select the
same token, including when it differs from the reference token.
The acceptance rate is the fraction of scored tokens accepted.
The following step always uses the reference continuation as
context. Each verification step ends at the first rejection or document
boundary and counts as one simulated target forward pass, with no
proposal-length cap.

The mean accepted draft length is the number of accepted draft tokens
per simulated target forward pass. To compare text coverage across
granularities, we also report accepted bytes per forward, summing the
decoded byte lengths of the accepted draft tokens. Both averages
include verification steps with no accepted tokens. Special tokens
contribute no bytes; correction and bonus outputs are excluded.

Related work uses similar diagnostics of draft--target agreement.
\citet{mamou2024dynamic} let the draft continue until the first
rejection, identifying the maximum accepted draft length at each
iteration without choosing a fixed lookahead or a stopping heuristic.
Their oracle follows generated histories, whereas ours measures
agreement on the given evaluation text.

EvoESAP \citep{liu2026evoesap} likewise evaluates acceptance on fixed
reference prefixes using teacher forcing. It averages the expected
single-token acceptance computed from the overlap between draft and
target next-token distributions. Our diagnostic measures consecutive
agreements and their byte coverage.

Figure~\ref{fig:draft-predictability} shows high acceptance rates in
the lowest-BPB intervals for both families. A much larger share of
\byte{} tokens lies in these intervals, and most of them are accepted.
Together with the longer byte acceptance lengths in
Table~\ref{tab:draft}, this supports using small drafts to supply many
low-BPB byte continuations for joint target verification.

\begin{figure}[!t]
  \centering
  \includegraphics[width=\linewidth]{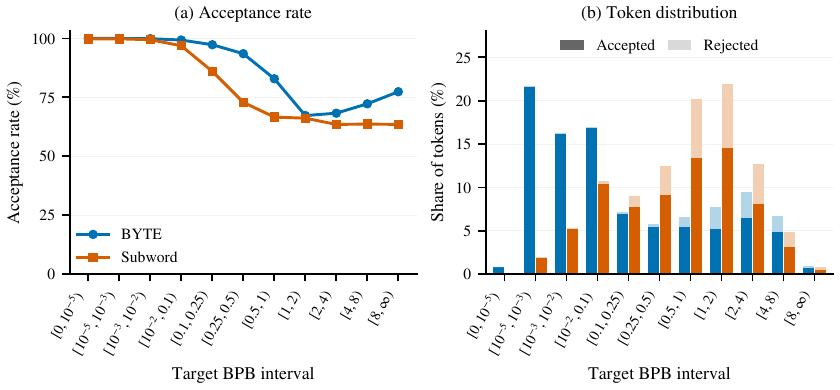}
  \caption{Small drafts achieve high acceptance rates on low-BPB byte tokens.
    Both families use 16-layer drafts and same-family 1B
    targets with verification on the given evaluation text
    within native windows.
    Left: acceptance rate within each target-BPB interval. Right: target-BPB
    distributions split into accepted (dark) and rejected (light) tokens,
    with each native token weighted equally. Colors identify model families.}
  \label{fig:draft-predictability}
\end{figure}

\paragraph{Implications for speculative decoding.}
\label{app:draft-interpretation}
The drafting experiment provides a proof of concept for computation
allocation at inference.
Longer byte agreement runs indicate that small drafts can match their
targets across many consecutive byte tokens in reference text. These
runs provide a teacher-forced diagnostic of drafting opportunities.
Generated histories may differ from the reference text, so these
lengths do not provide a general upper bound on accepted draft length
during generation. In speculative decoding
\citep{leviathan2023speculative}, accepted tokens become part
of the subsequent context, and each round includes correction or bonus
outputs. Establishing practical speedups therefore requires decoding
along generated histories with finite proposals and measuring the
costs of drafting and target verification.

\end{document}